\documentclass[10pt,twocolumn,letterpaper]{article}

\usepackage[pagenumbers]{iccv} 

\usepackage{algorithm, algorithmic}
\usepackage{amsmath, amssymb, amsthm, mathtools, bm}
\usepackage{enumitem}
\usepackage{booktabs} 
\usepackage{multirow} 
\usepackage{soul}

\theoremstyle{plain}                   
\newtheorem{theorem}{Theorem}[section] 
\newtheorem{lemma}[theorem]{Lemma}     
\newtheorem{proposition}[theorem]{Proposition}

\theoremstyle{definition}              
\newtheorem{definition}[theorem]{Definition}
\newtheorem{assumption}[theorem]{Assumption}

\theoremstyle{remark}

\newcommand{\1}{\mathbf{1}}           
          
\newcommand{\diag}{\operatorname{diag}}

\definecolor{iccvblue}{rgb}{0.21,0.49,0.74}
\usepackage[pagebackref,breaklinks,colorlinks,allcolors=iccvblue]{hyperref}

\def\paperID{6533} 
\def\confName{ICCV}
\def\confYear{2025}

\title{Falcon: Fractional Alternating Cut with Overcoming minima\\ in Unsupervised Segmentation}

\author{
  Xiao Zhang\textsuperscript{1} \quad
  Xiangyu Han\textsuperscript{1} \quad
  Xiwen Lai\textsuperscript{1} \quad
  Yao Sun\textsuperscript{2} \quad
  Pei Zhang\textsuperscript{3} \quad
  Konrad Kording\textsuperscript{1} \quad
  \vspace{2mm} \\ 
  $^{1}$University of Pennsylvania \quad 
  $^{2}$Hong Kong Polytechnic University \quad 
  $^{3}$Wuhan University \\
}

\begin{document}
\maketitle
\begin{abstract}
Today's unsupervised image segmentation algorithms often segment suboptimally. Modern graph-cut based approaches rely on high-dimensional attention maps from Transformer-based foundation models, typically employing a relaxed Normalized Cut solved recursively via the Fiedler vector (the eigenvector of the second smallest eigenvalue). Consequently, they still lag behind supervised methods in both mask generation speed and segmentation accuracy. We present a regularized fractional alternating cut (Falcon), an optimization-based K-way Normalized Cut without relying on recursive eigenvector computations, achieving substantially improved speed and accuracy. Falcon operates in two stages: (1) a fast K-way Normalized Cut solved by extending into a fractional quadratic transformation, with an alternating iterative procedure and regularization to avoid local minima; and (2) refinement of the resulting masks using complementary low-level information, producing high-quality pixel-level segmentations. Experiments show that Falcon not only surpasses existing state-of-the-art methods by an average of 2.5\% across six widely recognized benchmarks (reaching up to 4.3\% improvement on Cityscapes), but also reduces runtime by around 30\% compared to prior graph-based approaches. These findings demonstrate that the semantic information within foundation-model attention can be effectively harnessed by a highly parallelizable graph cut framework. Consequently, Falcon can narrow the gap between unsupervised and supervised segmentation, enhancing scalability in real-world applications and paving the way for dense prediction-based vision pre-training in various downstream tasks. 
The code is released in \url{https://github.com/KordingLab/Falcon}.

\end{abstract}   
\vspace{-5mm}
\section{Introduction}
\vspace{-3mm}
\label{sec:intro}
\begin{figure}[t]
    \centering
    \includegraphics[width=0.40\textwidth]{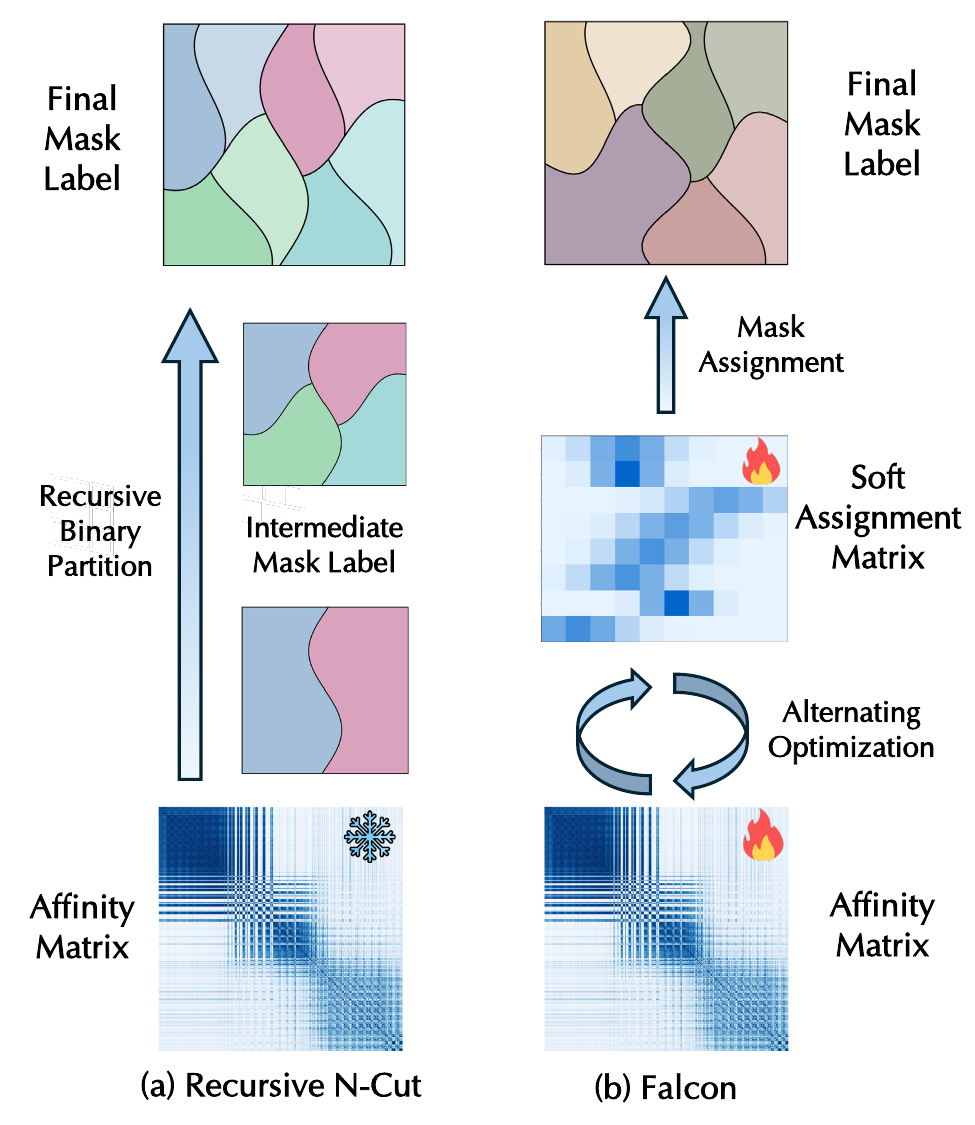}
    \vspace{-2mm}
    \caption{\textbf{Recursive N-Cut vs Falcon (ours).} Our method addresses the graph cut problem by alternately optimizing and regularizing both the soft assignment matrix and the affinity matrix, distinguishing it from recursive N-Cut methods.}
    \label{fig: Teaser figure}
    \vspace{-4mm}
\end{figure}

\begin{figure*}[t]
    \centering
    \includegraphics[width=0.92\textwidth]{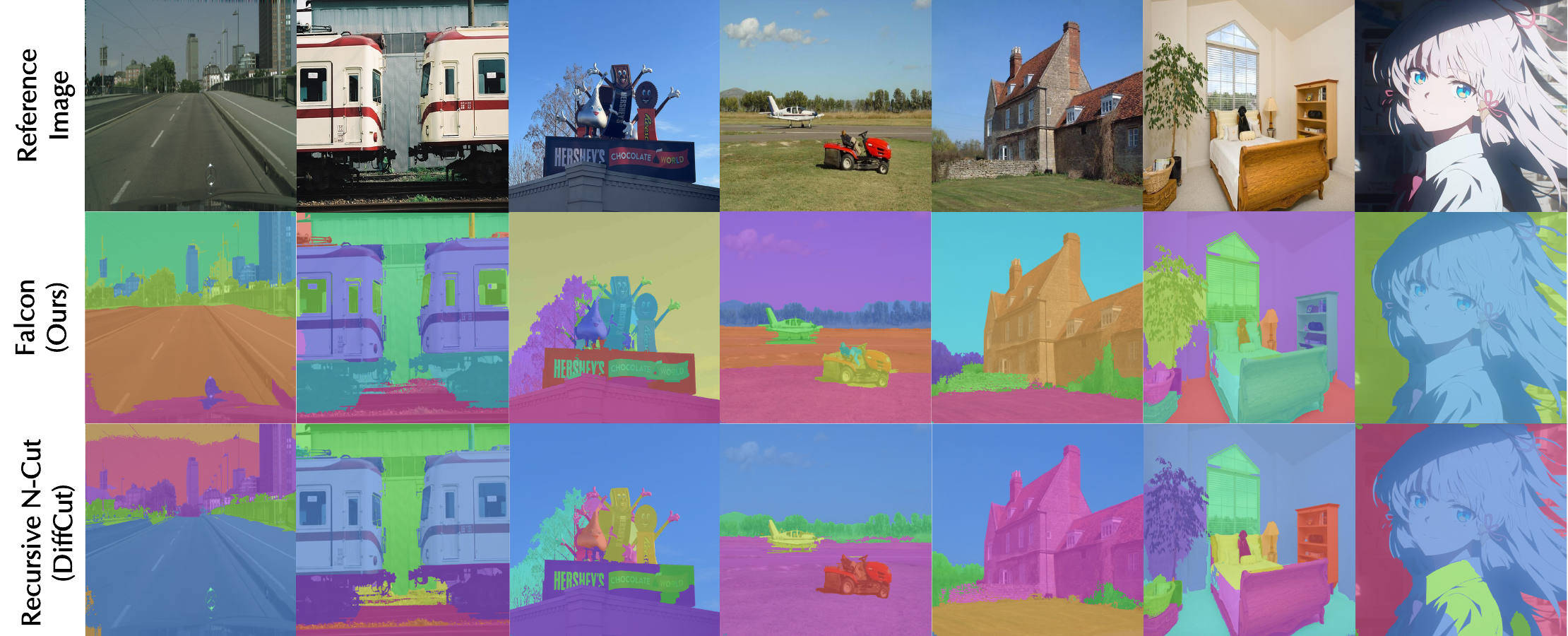}
    \vspace{-2mm}
    \caption{\textbf{Falcon Visual Segmentation Comparison.} Our \textbf{Falcon} method employs a fractional alternating optimized n‑Cut strategy, enhanced by multiple regularization techniques that effectively overcome local minima. Compared to \textbf{DiffCut}, \textbf{Falcon} reveals higher degree of fine details, for example, in the first image it distinguishes the car front, lanes, distant trees, and high-rise structures; in the second, it segments the intricate details of the train body; in the third, it separates billboards from rooftops; in the fourth, it isolates a child resting on a tractor; in the fifth, it clearly differentiates the castle’s surrounding walls and vines; in the sixth, it extracts items within a cabinet; in the final image, virtual environment with complex lighting conditions, \textbf{Falcon} robustly segments the complete human figure.}
    \label{fig: Results figure}
    \vspace{-4mm}
\end{figure*}

\quad Semantic segmentation partitions an image into regions whose pixels share the same semantics (i.e. being part of the same object), which matters for AI systems in terms of  perception, reasoning, planning, and acting in an object-centric manner~\cite{wang2023cut, zhu2016beyond}. As a critical and fundamental computer vision task, semantic segmentation underpins numerous downstream applications, including image editing, medical imaging, and autonomous driving~\cite{esser2023structure,zhou2019unet++, li2022hdmapnet, han2024extrapolated}. Semantic segmentation is a key computer vision task with countless downstream applications.

Semantic segmentation as rapidly improved over the last decade. Fully Convolutional Networks (FCNs) marked a major breakthrough in semantic segmentation~\cite{long2015fully}. Subsequently, numerous improved architectures—such as SegNet~\cite{hu2018squeeze}, U-Net~\cite{wang2021remote}, the DeepLab series~\cite{chen2017deeplab}, and PSPNet~\cite{zhao2017pyramid}—were developed, with a focus on capturing multi-scale features and fusing low-level and high-level information for more accurate segmentation outcomes. The advent of instance segmentation and panoptic segmentation further broadened the applicability of segmentation tasks, with Mask R-CNN~\cite{he2017mask} becoming a standard baseline~\cite{sapkota2024comparing, Zhang2020MaskRefinedRA}. Following the success of the Vision Transformer (ViT)~\cite{dosovitskiy2020image} in image classification, a variety of Transformer-based segmentation networks have emerged~\cite{oquab2023dinov2, caron2021emerging}, offering enhanced global modeling capabilities and opening new research avenues in image segmentation, such as SegFormer~\cite{xie2021segformer} Mask2Former~\cite{cheng2022masked}. More recently, the Segment Anything Model (SAM)~\cite{ravi2024sam2}, which leverages large-scale data (1.1B segmentation annotations), ViT-driven representations, and prompt-based segmentation, has introduced a new paradigm in computer vision. However, all these segmentation models require pixel-level annotations, which are both difficult and resource-intensive to obtain. 
           
 The difficulty to obtain good data has driven interest in unsupervised approaches, such as zero-shot unsupervised segmentation~\cite{tian2024diffuse, couairon2025diffcut}, where the goal is to segment images containing previously unseen categories—an inherently more challenging problem. Unsupervised image segmentation has recently advanced by incorporating self-supervised learning and traditional computer vision principles into deep learning pipelines~\cite{sick2024unsupervised}. Self-supervised learning (SSL) produces meaningful feature representations without requiring annotations~\cite{wen2022self, ziegler2022self}. For instance, STEGO~\cite{hamilton2022unsupervised} employs contrastive learning to extract patch-level features and refines segmentation masks through knowledge distillation and post-processing techniques such as Conditional Random Field (CRF). Another important approach incorporates traditional segmentation techniques such as clustering and graph-based optimization into deep learning pipelines~\cite{tian2024diffuse}. Invariant Information Clustering (IIC)~\cite{ji2019invariant} and PiCIE~\cite{cho2021picie} formulate segmentation as an unsupervised clustering problem, enforcing invariance and equivariance constraints to group pixels into semantically consistent regions. These algorithms use self supervised learning but no explicit cutting operations.

Other methods do use cutting operations. Graph-based methods such as TokenCut~\cite{wang2023tokencut} and MaskCut~\cite{wang2023cut} recursively leverage Normalized Cut (N-Cut) on deep feature representations. More recently, DiffCut~\cite{,couairon2025diffcut} integrates graph optimization with diffusion models to enhance connectivity between pixels, leading to improved segmentation consistency. These methods demonstrate the effectiveness of combining self-supervised representations with graph-based formulations. However, these approaches, as shown in \autoref{fig: Teaser figure} that recursively partition the feature space via a relaxed Normalized Cut and the Fiedler vector~\cite{shi2000normalized,ng2001spectral} often face three major limitations. First, their hierarchical modeling tends to be suboptimal: each recursive step is a greedy partition that cannot be globally refined. Second, by relaxing the N-Cut~\cite{shi2000normalized} objective, the segmentation procedure is prone to locally optimal solutions, especially when relying on the second eigenvector for a strict binary split. Third, the repeated eigen-decomposition across scales increases computational overhead, resulting in slower inference. We may believe that overcoming these three problems should increase speed and improve solutions.

To tackle these challenges, we propose Falcon, a \textbf{F}ractional \textbf{Al}ternating optimized N-\textbf{C}ut~\cite{shi2000normalized} with multiple regularization technologies for \textbf{O}vercoming local mi\textbf{n} ima. This novel approach reimagines unsupervised image segmentation using tokens from pre-trained self-supervised transformers. Instead of relying on recursive binary partitioning through the Fiedler vector, Falcon introduces a regularized, parallelizable K-way Normalized Cut formulation that effectively addresses the three major limitations of existing graph-cut methods.

Falcon operates in two stages. First, our fast K-way Normalized Cut algorithm processes the attention maps from vision foundation models to generate semantically meaningful low-resolution patch-level segmentations. Unlike previous approaches that process each cut sequentially, our method optimizes all segments simultaneously, resulting in more coherent global segmentation. Second, we introduce a refinement stage that leverages complementary low-level information (RGB and/or depth ) to enhance the resolution and precision of the generated masks, producing high-quality pixel-level segmentations that better capture object boundaries and fine-grained details as shown in~\autoref{fig: Results figure}.

Our contributions can be summarized as follows.
\begin{enumerate}
    \item We analyze several deficiencies in recursive binary N-Cut~\cite{shi2000normalized} based on ViT~\cite{dosovitskiy2020image} tokens that lead to sub-optimal results: (a) distances between high-dimensional tokens become less discriminative (the so-called “curse of dimensionality”), causing similarity uniformity where true data structures become obscured; (b) spectral relaxation introduces a gap, creating risks of local sub-optimality during discretization; (c) Recursive binary partitioning is a greedy algorithm where each step employs a locally optimal strategy that cannot guarantee global optimality nor allow for subsequent refinement.
    \item We introduce a fractional quadratic optimization objective for K-way N-Cut~\cite{shi2000normalized} and develop an efficient alternating optimization algorithm with regularization techniques, enabling parallel optimization of multiple segments. This effectively avoids the sub-optimality common in recursive greedy algorithms and spectral clustering relaxation based on the Fiedler eigenvector method. Additionally, we incorporate a power transformation when computing the affinity matrix to mitigate similarity uniformity caused by the curse of dimensionality.
    \item In experiments on challenging benchmarks, Falcon consistently outperforms prior methods – for instance, it improves mean IoU by  4.3 percentage points on most challenge dataset  Cityscapes~\cite{cordts2016cityscapes}. At the same time, it reduces segmentation runtime by 30\% relative to the spectral clustering baseline, significantly boosting efficiency.
\end{enumerate}

\vspace{-3mm}
\section{Related Works}
\label{sec:related}

\textbf{Vision Foundation Models.}
Vision foundation models typically leverage unlabeled data to learn robust and generalizable representations. Early contrastive methods like MoCo~\cite{he2020momentum} and BYOL~\cite{grill2020bootstrap} laid the groundwork for advanced frameworks such as SwAV~\cite{goyal2021self}, DINO~\cite{caron2021emerging,oquab2023dinov2,darcet2023vitneedreg}, and iBOT~\cite{zhou2021ibot}, while masked autoencoders~\cite{he2022masked} have further refined reconstruction-based pre-training. Beyond purely visual approaches, multimodal pre-training has surged in prominence, with models like CLIP~\cite{radford2021learning}, BLIP~\cite{li2022blip}, and Siglip~\cite{zhai2023sigmoid,tschannen2025siglip} aligning high-level image features to text. In parallel, diffusion-based methods such as Stable Diffusion~\cite{gupta2024progressive} extend these capabilities by learning rich generative representations, enabling tasks ranging from zero-shot classification to semantic correspondence. Collectively, these developments highlight the efficacy of foundation models in scaling to large, diverse datasets and adapting readily to downstream tasks.

\noindent\textbf{Semantic Segmentation.}
Semantic segmentation partitions an image into semantically coherent regions by labeling each pixel, enabling the understanding of the structured scene. It is broadly categorized into supervised and unsupervised methods. Supervised segmentation, extensively studied and achieving high accuracy~\cite{namekata2024emerdiff, wang2021max, cheng2021per, ravi2024sam2}, relies on large-scale annotated datasets. Recent work has explored text-based supervision to mitigate the need for dense annotations~\cite{ranasinghe2023perceptual, xu2022groupvit, cha2023learning, ren2023viewco}. In contrast, unsupervised methods often require dataset-specific training to achieve competitive performance~\cite{liang2023open, feng2023network, cho2021picie, li2023acseg}, and zero-shot segmentation for unseen categories remains challenging. DiffSeg~\cite{tian2024diffuse} leverages self-attention maps from a pre-trained diffusion model, applying KL-divergence-based iterative merging for segmentation. DiffCut~\cite{couairon2025diffcut} improves upon this by extracting richer encoder features from the self-attention block of a Transformer and incorporating a recursive N-Cut~\cite{ncut2000} algorithm.

\noindent\textbf{Graph-based Segmentation.}
Early approaches, such as Normalized Cut~\cite{ncut2000}, optimized the segmentation problem through spectral graph theory. \textit{et al.}\cite{specclus2007} formalized spectral clustering theory based on the N-Cut objective, yet its computational limitations persisted. To improve efficiency and adaptability, \cite{efficient2004} proposed an adaptive merging strategy using intra-region and inter-region criteria, while \cite{ranwalk2006} introduced a probabilistic random walk framework that leverages adjacent pixel relationships to handle complex textures and weak boundaries. Recent deep learning-integrated approaches, such as TokenCut\cite{tokencut2022}, compute token-level similarities via self-supervised Transformer features but remain constrained by N-Cut’s recursive bisection strategy and hard segmentation constraints.

\vspace{-3mm}

\begin{figure*}[ht]
    \centering
    \includegraphics[width=0.92\textwidth]{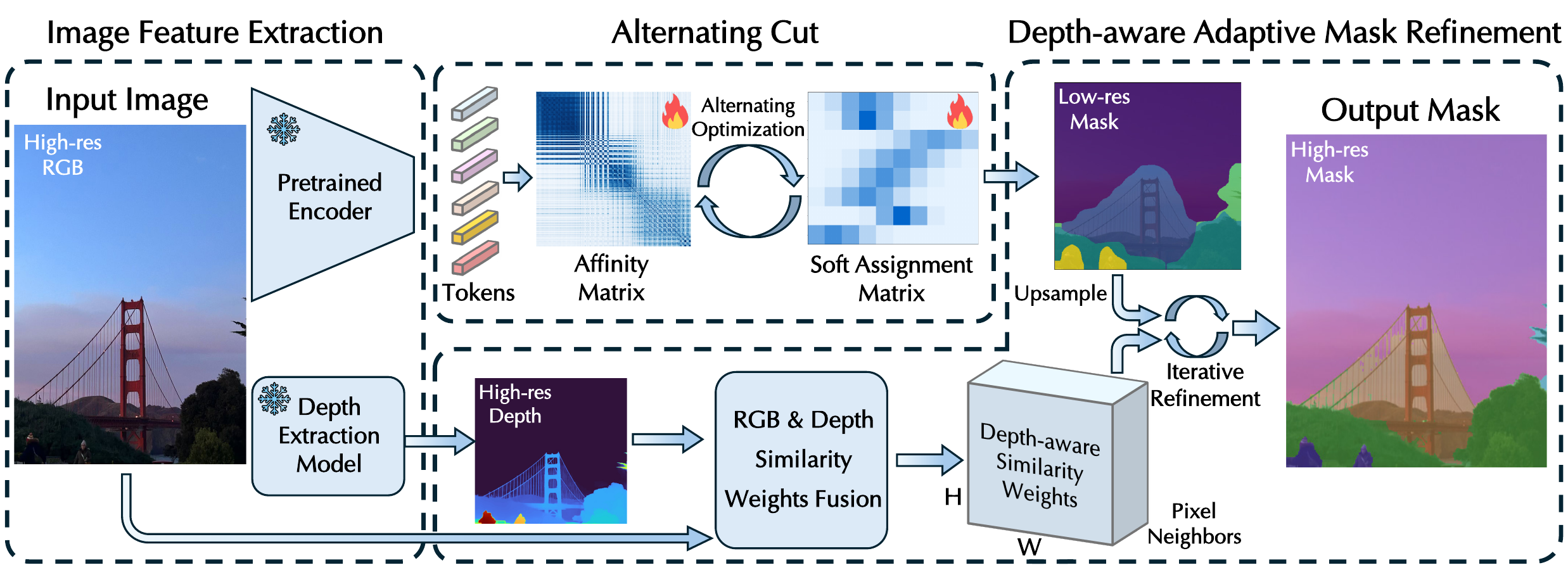}
    \caption{\textbf{Overview of Falcon.} (1) Image feature extraction: we extract tokens and depth map from the input image. (2) Alternating Cut: we construct affinity matrix between tokens and alternately optimize it with the soft assignment matrix. (3) Depth-aware Adaptive Mask Refinement: depth map and original RGB image flows into the similarity weights fusion module and produce the weights which is used to iterative refine the mask assignment obtained from Alternating Cut step.}
    \vspace{-2mm}
    \label{fig: Pipeline}
    \vspace{-3mm}
\end{figure*}

\section{Methodology}

\subsection{Revisit Recursive N-Cut on Token}
We revisit recursive Normalized Cut (N-Cut)~\cite{shi2000normalized} for token-based representations from high-dimensional embeddings (e.g., self-supervised transformers). Despite its success in graph-based segmentation, recursive N-Cut~\cite{shi2000normalized} struggles with (1) similarity concentration, (2) spectral relaxation gaps, and (3) sub-optimal recursive partitioning, degrading segmentation in complex, multi-scale images. We analyze these issues, highlighting the need for advanced methods.

\noindent\textbf{Similarity Concentration in High-Dimensional Spaces.}
When vectors reside in high-dimensional spaces, their pairwise cosine similarities tend to cluster tightly around zero~\cite{donoho2000high, vershynin2018high, Zhang_2019_CVPR}, a phenomenon commonly termed the “curse of dimensionality”~\cite{bellman2015applied, zimek2012survey}. As dimensionality grows, inter-point distances collapse into a narrow numerical range~\cite{beyer1999nearest}, giving rise to \emph{hubness}, where certain points (so-called “hubs”) become nearest neighbors to disproportionately many others~\cite{radovanovic2010hubs}. This uniformity in the pairwise affinity matrix obscures true clusters, diminishing the effectiveness of graph-based approaches like spectral clustering and normalized cut~\cite{radovanovic2010hubs}. Consequently, the interplay of cosine similarity concentration, measure concentration~\cite{ledoux2001concentration}, and hubness leads to blurred boundaries and weaker cluster structures in high-dimensional embedding spaces, creating fundamental barriers for downstream tasks and beyond.

\noindent\textbf{Gap in Spectral relaxation.}
The N-Cut~\cite{shi2000normalized} objective in graph-based segmentation seeks to partition a weighted graph into two disjoint sets by minimizing a normalized measure, a problem that is NP-hard under discrete constraints. To address this, the discrete indicator vector is relaxed to a continuous domain, enabling the problem to be reformulated as a Rayleigh quotient on the graph Laplacian whose solution is obtained via eigen-decomposition (typically using the second eigenvector). However, this continuous solution does not enforce the original binary constraints, necessitating thresholding or clustering that can introduce a relaxation gap between the continuous optimum and the discrete solution. See appendix for the mathematical details.

\noindent\textbf{Sub-optimality in Recursive Partitioning.}
Although recursive bipartitioning---where a two-way normalized cut is computed at each stage and then applied to each subgraph until \(K\) subsets are formed---is computationally appealing, it often leads to sub-optimal solutions. In particular, multi-way Ncut~\cite{shi2000normalized} does not possess an optimal substructure that guarantees local two-way cuts can be combined to produce a global optimum. Each bipart step, even if locally minimizing the normalized cut on a subgraph, permanently fixes a partition boundary that can deviate from the optimal multi-cut minimum and cannot be revoked in subsequent steps. Additionally, real-world graphs exhibit parallel communities rather than strictly nested ones, making layer-by-layer splitting mismatched to the data structure. Furthermore, relying on only the second smallest eigenvector of the graph Laplacian (the Fiedler vector) at each bipart step ignores additional eigenvectors that could reveal more nuanced multi-cluster separations. While recursive partitioning is easy to implement, it can incur larger overall cut values compared to algorithms that consider global \(K\)-way objectives or use multiple spectral components simultaneously.

\subsection{Segmentation as Graph Cut on Tokens}

Graph-based segmentation formulates image segmentation as a graph partitioning problem. Following the general pipeline of normalized cut-based methods~\cite{tokencut2022,wang2023cut,couairon2025diffcut}, we define an undirected graph \( G = (V, E) \) with \( N \) nodes, where each node corresponds to a \( d \)-dimensional patch token vector extracted from a vision Transformer. The node set is defined as:
\begin{equation}
    V = \{ \bm{f}_1, \bm{f}_2, \dots, \bm{f}_N \}, \quad \bm{f}_i \in \mathbb{R}^{d}.
\end{equation}
The edge set is constructed as:
\begin{equation}
    E = \{ (\bm{f}_i, \bm{f}_j) \mid i \neq j \}.
\end{equation}

To model the pairwise relationships between tokens, we define an affinity matrix \( \bm{W} \in \mathbb{R}^{N \times N} \) that encodes their similarities. The feature matrix is constructed as:
\begin{equation}
    \bm{F} = \begin{bmatrix} \bm{f}_1 & \bm{f}_2 & \cdots & \bm{f}_N \end{bmatrix}^T \in \mathbb{R}^{N \times d}.
\end{equation}
The raw affinity matrix is computed as:
\begin{equation}
    \bm{W}_{\text{raw}} = \bm{F} \bm{F}^T.
\end{equation}
To ensure numerical stability and normalize the values within \([0,1]\), we apply min-max normalization:
\begin{equation}
    \bm{W}_{\text{norm}} = \frac{\bm{W}_{\text{raw}} - \min(\bm{W}_{\text{raw}})}{\max(\bm{W}_{\text{raw}}) - \min(\bm{W}_{\text{raw}})}.
\end{equation}
We then apply power transformation and regularization to obtain the final affinity matrix:
\begin{equation}
    \bm{W} = \bm{W}_{\text{norm}}^\alpha + \lambda \operatorname{diag}(\bm{D}),
\end{equation}
where \( \lambda \) is a regularization coefficient, and the degree matrix is defined as:
\begin{equation}
    \bm{D} = \operatorname{diag}(d_1, d_2, \dots, d_N), \quad d_i = \sum_{j} {W}_{ij}.
    \vspace{-1.5mm}
\end{equation}
The power transformation technique was originally introduced in WGCNA~\cite{langfelder2008wgcna} for transcriptomics. More recently, it has been adopted in token-based graph cut methods~\cite{tokencut2022,couairon2025diffcut}. In Sec.~\ref{ssec:power_trans}, we analyze its effect from both theoretical and experimental perspectives.

Given this formulation, we aim to partition the graph into \( K \) disjoint subsets \( \{ P_1, P_2, \dots, P_K \} \), such that intra-partition similarity is maximized while inter-partition connectivity is minimized. This is formulated as a \( K \)-way Normalized Cut (N-Cut)~\cite{shi2000normalized} problem:
\vspace{-1.5mm}
\begin{equation}
    \operatorname{Ncut}(P_1,\dots,P_K) = \sum_{k=1}^{K} \frac{\operatorname{cut}(P_k, \bar{P}_k)}{\operatorname{vol}(P_k)}.
    \vspace{-1.5mm}
\end{equation}

The cut cost measures the total edge weight between nodes in \( P_k \) and those outside \( P_k \):
\vspace{-1.5mm}
\begin{equation}
    \operatorname{cut}(P_k, \bar{P}_k) = \sum_{i \in P_k, j \notin P_k} W_{ij}.
    \vspace{-1.5mm}
\end{equation}
The volume of a partition is defined as:
\vspace{-1.5mm}
\begin{equation}
    \operatorname{vol}(P_k) = \sum_{i \in P_k} d_i = \bm{x}_k^T \bm{D} \bm{x}_k,
    \vspace{-1.5mm}
\end{equation}
where \( \bm{x}_k \in \mathbb{R}^{N} \) is a partition indicator vector:
\begin{equation}
    (\bm{x}_k)_i =
    \begin{cases}
        1, & \text{if node } i \text{ belongs to partition } k, \\
        0, & \text{otherwise}.
    \end{cases}
\end{equation}

To express the cut cost in matrix form, we use the graph Laplacian \( \bm{L} = \bm{D} - \bm{W} \) and rewrite:
\begin{equation}
    \operatorname{cut}(P_k, \bar{P}_k) = \bm{x}_k^T \bm{L} \bm{x}_k.
\end{equation}
Thus, the $K$-way N-Cut objective becomes:
\vspace{-3mm}
\begin{equation}
    \operatorname{Ncut}(P_1,\dots,P_K) = \sum_{k=1}^{K} \frac{\bm{x}_k^T \bm{L} \bm{x}_k}{\bm{x}_k^T \bm{D} \bm{x}_k}.
    \label{eq:ncut_rqmax}
    \vspace{-1.5mm}
\end{equation}

Let \( \bm{X} = [\bm{x}_1, \bm{x}_2, \dots, \bm{x}_K] \in \mathbb{R}^{N \times K} \) be the partition assignment matrix, where each row sums to 1 in the relaxed case. The N-Cut problem is then rewritten as the equivalent Rayleigh quotient maximization (see Appendix~\ref{app:ncut}):
\vspace{-1.5mm}
\begin{equation}
    \max_{\bm{X}} \sum_{k=1}^{K} \frac{\bm{x}_k^T \bm{W} \bm{x}_k}{\bm{x}_k^T \bm{D} \bm{x}_k}.
    \vspace{-1.5mm}
\end{equation}
To prevent trivial solutions, we impose the constraint:
\vspace{-1.5mm}
\begin{equation}
    \bm{X}^\top \bm{D} \bm{X} = \bm{I}_K.
    \vspace{-1.5mm}
\end{equation}
This ensures that the solution maintains an orthogonality constraint, preventing partitions from collapsing into a degenerate solution.

\subsection{Fractional Alternating Cut with Regularization}

To efficiently solve the optimization problem, following the fractional programming thought in~\cite{NEURIPS2024_a3017a8d,9706351}, we introduce auxiliary variables \( y_k \) and apply a quadratic transform~\cite{NIPS2001_a0128693} (see Appendix~\ref{app:frac_transform}), leading to the reformulated optimization objective from Eq.\ref{eq:ncut_rqmax}:
\vspace{-3mm}
\begin{equation}
    \max_{X, y} \sum_{k=1}^{K} \left( 2y_k \sqrt{\bm{x}_k^T \bm{W} \bm{x}_k} - y_k^2 \bm{x}_k^T \bm{D} \bm{x}_k \right).
    \vspace{-3mm}
    \label{eq:ncut_qt}
\end{equation}

Here, the quadratic transform eliminates the fractional structure in the original objective, making the optimization more tractable. Instead of directly optimizing the Eq.\ref{eq:ncut_rqmax}, we alternately update \( \bm{X} \) and \( y_k \) in Eq.\ref{eq:ncut_qt}, ensuring a smooth and efficient optimization process.

By taking the derivative of the objective function Eq.\ref{eq:ncut_qt} with respect to \( y_k \) and setting it to zero, the optimal closed-form solution for \( y_k \) is obtained as:
\vspace{-1.5mm}
\begin{equation} \label{eq:yk-opt}
    y_k = \sqrt{\frac{\bm{x}_k^T \bm{W} \bm{x}_k}{\bm{x}_k^T \bm{D} \bm{x}_k}}.
    \vspace{-1.5mm}
\end{equation}

Starting with an initial random soft assignment matrix \( \bm{X} \), the optimization proceeds by alternately updating \( y_k \), refining the partition assignments, and adjusting the affinity matrix until convergence. Given \( y_k \), we update the assignment matrix \( \bm{X} \) as follows:
\vspace{-1.5mm}
\begin{equation}
    X_{ik}^{\text{new}} = \operatorname{Softmax}_k \left( \frac{\sum_{j} {W}_{ij} {X}^{\text{old}}_{jk}}{\sum_j {X}^{\text{old}}_{jk} {D}_{jj}} y_k \right).
    \vspace{-1.5mm}
\end{equation}
This update encourages each node \( i \) to adjust its assignment probability towards partitions with stronger affinities, weighted by the auxiliary variable \( y_k \), which reflects the relative importance of each cluster.

To further refine the segmentation, the affinity matrix \( \bm{W} \) is dynamically adjusted using cosine similarity to better capture the underlying data structure (see Appendix~\ref{app:graph_update}):
\begin{equation}
    {W}_{ij}^{\text{new}} = {W}_{ij}^{\text{old}} \cdot \exp\left[-\frac{(1 - \cos_{ij})^2}{\beta} \right],
    \vspace{-1.5mm}
\end{equation}
where \( \cos_{ij} = \frac{\langle \bm{X}_i, \bm{X}_j \rangle}{\|\bm{X}_i\| \|\bm{X}_j\|} \). This adjustment enhances the separation between clusters by penalizing edges connecting dissimilar nodes while reinforcing intra-cluster connections. Once \( \bm{W}^{\text{new}} \) is updated, the algorithm iterates back to the first step, recomputing \( y_k \) and continuing the alternating process. The iterative updates of \( y_k \) and \( \bm{X} \) allow the optimization to progressively refine both the segmentation and the underlying graph representation, ensuring convergence to a stable solution.

\subsection{Segmentation Mask Generation}

Given the final soft assignment matrix \( \bm{X} \), we first compute an initial segmentation mask in low resolution by assigning each node to the partition with the highest probability:
\vspace{-1mm}
\begin{equation}
    \operatorname{mask}_{\text{low}}(i) = \arg\max_{k} X_{ik}.
    \vspace{-3mm}
\end{equation}
Since the segmentation mask is computed on a coarse resolution corresponding to the tokenized representation of the input, it is necessary to upsample the mask to the original image resolution. We achieve this by applying nearest-neighbor interpolation to obtain $\operatorname{mask}_{\text{high}}$. Once the high-resolution mask is obtained, we compute a partition-wise feature representation to refine the segmentation. Specifically, for each partition \( k \), the feature center is obtained by averaging the feature vectors of all pixels belonging to that partition. Using the upsampled feature map \( Z \), the feature center of partition \( k \) is computed as:
\vspace{-3mm}
\begin{equation}
    C_k = \frac{\sum_{h,w} M_{k,h,w} \cdot Z_{:,h,w}}{\sum_{h,w} M_{k,h,w}},
    \vspace{-1.5mm}
\end{equation}
where \( M_{k,h,w} \) is the one-hot encoded mask that indicates whether pixel \( (h,w) \) belongs to partition \( k \). This ensures that \( C_k \) represents the average feature vector of all pixels assigned to partition \( k \), capturing the characteristic features of that partition.

With the partition feature centers computed, the final segmentation mask is refined by reassigning each pixel based on its similarity to the partition embeddings. The similarity between the feature vector at pixel \( (h,w) \) and each partition center \( C_k \) is measured by the dot product:
\begin{equation}
    S_{h,w,k} = Z_{:,h,w}^T C_k.
\end{equation}
Each pixel is assigned to the highest similarity partition:
\vspace{-1mm}
\begin{equation}
    \operatorname{mask}_{\text{final}}(h,w) = \arg\max_{k} S_{h,w,k}.
    \vspace{-3mm}
\end{equation}

By leveraging both the initial token-level assignment and the refined partition-wise feature representation, this approach ensures a segmentation mask that aligns well with the image structure while preserving consistency in local feature distributions.

\begin{figure}[t]
    \centering
    \includegraphics[width=0.38\textwidth]{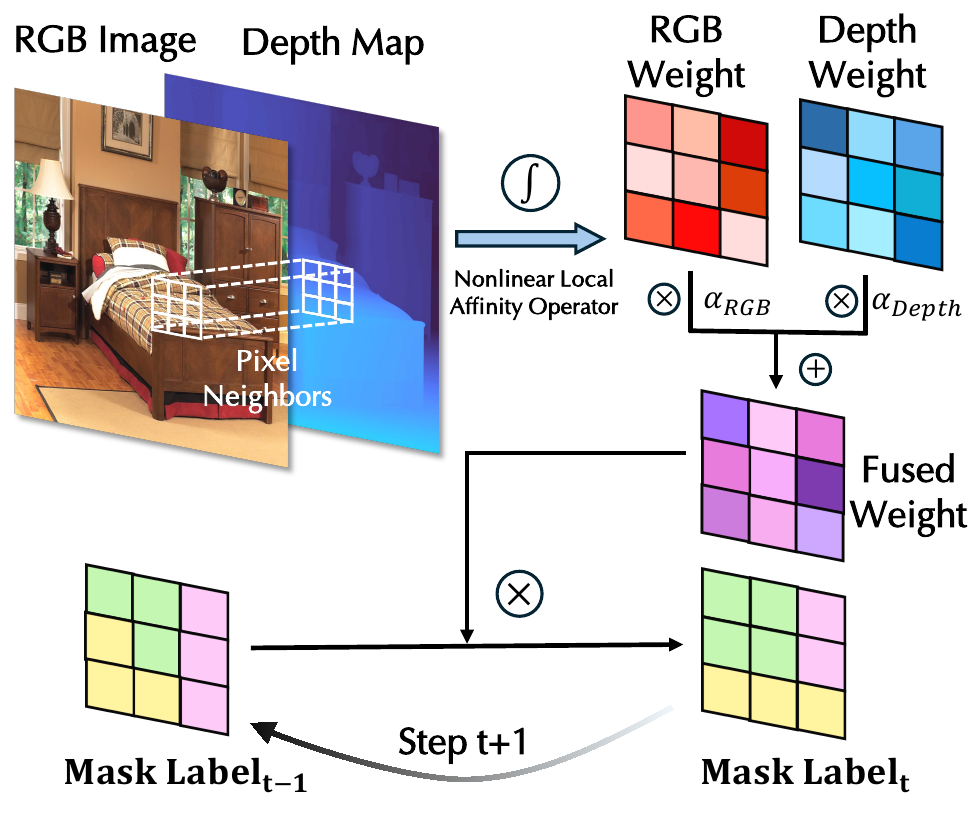}
    \vspace{-2mm}
    \caption{\textbf{Depth-aware Non-linear Adaptive Mask Refinement (DREAM).} RGB and depth Similarity weight matrix are constructed based on the affinity between current pixel and its neighbors, and are fused through the blending weights. Then the mask label iteratively updates based on the fused weights.}
    \label{fig: DREAM figure}
    \vspace{-2mm}
\end{figure}

\subsection{Depth-aware Non-linear Adaptive Mask Refinement (DREAM)}

Unlike previous works that refine segmentation masks by propagating information solely in the RGB domain using PAMR~\cite{Araslanov:2020:SSS}, we propose \textit{\textbf{D}epth-awa\textbf{r}e Non-lin\textbf{e}ar \textbf{A}daptive \textbf{M}ask Refinement (DREAM)}, as shown in \autoref{fig: DREAM figure}. This post-processing technique refines segmentation masks by leveraging local non-linear feature weighting in both RGB and depth modalities. \textit{DREAM} can iteratively refine the mask by computing local weighting measures and propagating segmentation labels in a depth-aware manner.

To compute these local weighting measures, we first define a non-linear operator based on feature dissimilarities. Given a feature map \( \phi \in \mathbb{R}^{B \times C \times H \times W} \), the measure is computed over an 8-connected neighborhood. Instead of using a simple linear difference, we apply a non-linear transformation using the Exponential Linear Unit (ELU):
\vspace{-1.5mm}
\begin{equation}
    \Omega(\phi) = \sum_{(i,j) \in \mathcal{N}} \bigl(\phi_{i,j} - \phi_{c}\bigr)
                   + \lambda \,\operatorname{ELU} \!\bigl(\phi_{i,j} - \phi_{c}\bigr).
                   \vspace{-1.5mm}
\end{equation}

Here $\phi_{c}$ is the center pixel, \( \mathcal{N} \) denotes the eight neighboring pixels of a given center pixel, and \( \lambda \) controls the degree of nonlinearity. The ELU transformation ensures that small differences are preserved while amplifying significant variations, making the weighting measure more robust to local feature inconsistencies.

In addition to this non-linear measure, we compute the local standard deviation over a \( 3 \times 3 \) neighborhood to quantify spatial variations in feature maps:

\vspace{-1.5mm}
\begin{equation}
    \sigma(\phi) = \sqrt{\frac{\sum\limits_{(i,j) \in \mathcal{N}} \bigl(\phi_{i,j} - \bar{\phi}\bigr)^2}{|\mathcal{N}|} 
    },
    \vspace{-1.5mm}
\end{equation}

where \( \bar{\phi} \) is the mean feature value over \( \mathcal{N} \). The standard deviation quantifies local texture variability, which helps refine segmentation masks in regions with strong structural changes by preventing excessive propagation.

To balance contributions from different modalities, we scale \(\Omega(\phi)\) via depth-aware standardization:
\vspace{-1.5mm}
\begin{equation}
    \Omega_{\text{std}}(\phi) = -\frac{\Omega(\phi)}{\eta \cdot \sigma(\phi)},
    \vspace{-3mm}
\end{equation}
where \( \eta \) can control the impact of local standard deviation, ensuring the weighting measure remains well-scaled across different feature distributions. Given an input RGB feature map \( \phi_{\text{rgb}} \in \mathbb{R}^{B \times 3 \times H \times W} \) and its corresponding depth feature map \( \phi_{\text{depth}} \in \mathbb{R}^{B \times 1 \times H \times W} \), we compute the non-linear local weighting measure separately for each modality as $\Omega_{\text{rgb}} = \Omega_{\text{std}}\bigl(\phi_{\text{rgb}}\bigr)$ and $\Omega_{\text{depth}} = \Omega_{\text{std}}\bigl(\phi_{\text{depth}}\bigr)$ respectively. The final combined measure is obtained by merging the RGB- and depth-based components:

\vspace{-1.5mm}
\begin{equation}
    \Omega = \alpha_{\text{rgb}}\,\Omega_{\text{rgb}} \;+\; \alpha_{\text{depth}}\,\Omega_{\text{depth}}.
\end{equation}

This formulation allows the refinement process to incorporate both color and geometric cues, making segmentation masks more robust to ambiguous textures and depth discontinuities. Starting with an initial mask \( M \), the refined mask is computed iteratively using the merged measure:
\vspace{-1.5mm}
\begin{equation}
    M^{(t+1)} = \sum_{(i,j) \in \mathcal{N}} M^{(t)}_{i,j} \cdot \Omega_{i,j}.
    \vspace{-3mm}
\end{equation}

At each iteration \( t \), the mask values are updated by diffusing information based on the local weighting measure, encouraging smooth and structure-aware refinements. This iterative process continues until convergence, ensuring that segmentation masks align more accurately with underlying image structures.

Our overall pipeline is shown in \autoref{fig: Pipeline}.
\vspace{-3mm}
\section{Experiments}
\label{sec:exp}

\begin{table*}[t]
    \centering
    \begin{tabular}{lcccccc}
        \toprule      
        \textbf{Model} & \textbf{VOC} & \textbf{Context} & \textbf{COCO-Object} & \textbf{COCO-Stuff-27} & \textbf{Cityscapes} & \textbf{ADE20K} \\
        \midrule
        \textbf{MaskCLIP~\cite{dong2023maskclip}} &38.8&23.6&20.6&19.6&10.0&9.8 \\
        \midrule
        \textbf{MaskCut~\cite{wang2023cut}} & 53.8 & 43.4&30.1&41.7&18.7&35.7  \\
        \textbf{DiffSeg~\cite{tian2024diffuse}} & 49.8 & 48.8 & 23.2 & 44.2 & 16.8 & 37.7 \\
        \textbf{DiffCut~\cite{couairon2025diffcut}} & 65.2 & 56.5 & 34.1 & 49.1 & 30.6 & 44.3 \\
        \midrule
        \textbf{Falcon (ours)} & \textbf{66.6} & \textbf{57.8} & \textbf{35.8} & \textbf{52.6$^{*}$} & \textbf{34.9} & \textbf{47.1}\\
        \bottomrule
    \end{tabular}
    \caption{\textbf{Comparison of unsupervised segmentation methods across benchmarks.} Our Falcon achieves the highest mIoU on all datasets. Note that MaskCLIP~\cite{dong2023maskclip} requires text input to guide segmentation. The notation $^{*}$ indicates no mask refinement.}
    \label{tab:results}
    \vspace{-3mm}
\end{table*}

\begin{table*}[h]
  \centering
  \begin{tabular}{@{}llccccc@{}}
    \toprule
    \multicolumn{2}{@{}c}{\textbf{Settings}}& \textbf{Context} & \textbf{COCO-Object} & \textbf{COCO-Stuff} & \textbf{Cityscapes} & \textbf{ADE20K} \\
    \midrule
    \multicolumn{2}{@{}c}{\textbf{Our Base}} & 51.2 & 33.0 & 49.5 & 31.3 & 40.5   \\
    \multicolumn{2}{@{}c}{\textbf{+ Soft Assignment}} & 53.1 & 32.8 & 52.7 & 32.1 &  40.0  \\
    \multicolumn{2}{@{}c}{\textbf{+ Dynamic Affinity Regularization}} & 55.3 & 33.3 & \textbf{52.6} & 32.2 &  42.3  \\
    \midrule
    \multirow{3}{*}{\textbf{+ Mask Refinement}} & \textbf{PAMR} & 56.6 & 34.7 & 50.6 & 32.1 & 46.3 \\
    & \textbf{DREAM (RGB)} & 57.0 & 35.5 & 51.6 & 34.2 & 46.4 \\
    & \textbf{DREAM (RGBD)} & \textbf{57.8} & \textbf{35.8} & 51.8 & \textbf{34.9} & \textbf{47.1} \\
    \bottomrule
  \end{tabular}
  \caption{\textbf{Ablation study on mIoU across datasets.} We evaluate the impact of incrementally adding Soft Assignment, Dynamic Affinity Regularization, and Mask Refinement. The last stage compares three refinement methods: PAMR, DREAM (RGB), and DREAM (RGBD).}
  \label{tab:ablation_miou_cumulative}
  \vspace{-3mm}
\end{table*}

\textbf{Implementation details.}
Following DiffCut, we build on a distilled Stable Diffusion model~\cite{gupta2024progressive} (without text prompts, features extracted at $t=10$) and integrate Depth-Pro~\cite{bochkovskii2024depth} for depth extraction. Our pipeline performs feature normalization ($\ell_{2}$ norm) at a lower resolution $32\times 32$ (also the initial number of partitions) and then applies Falcon to generate an initial segmentation. We set the power transformation parameter to $4.5$ (except ADE20K, where $\alpha = 5.5$). Once low-resolution clusters are obtained, we upsample these segment labels to a final mask size of $128\times128$, using the corresponding upsampled features for refinement. For certain experiments, we optionally employ DREAM to further clean up boundaries and integrate depth information. All experiments are conducted in PyTorch, using a single NVIDIA RTX 4090 GPU. Our method supports a range of common datasets by unifying them under the same segmentation pipeline. Resizing each image to $1024\times 1024$ and performing both clustering and optional refinement steps take around $0.6$ seconds per image in practice.

\noindent\textbf{Datasets.}
We evaluate Falcon on six widely used benchmark datasets to ensure a fair comparison. Pascal VOC~\cite{everingham2015pascal} consists of 20 foreground object classes and is a standard benchmark for object detection and segmentation. Pascal Context~\cite{mottaghi2014role}, an extension of Pascal VOC, expands the dataset to 59 foreground classes with additional contextual elements. COCO-Object~\cite{lin2014microsoft} is a subset of the COCO dataset containing 80 distinct object categories, while COCO-Stuff-27 consolidates 80 “thing” categories and 91 “stuff” categories into 27 mid-level semantic classes. Cityscapes~\cite{cordts2016cityscapes} focuses on semantic segmentation of urban street scenes and includes 27 foreground classes. ADE20K~\cite{zhou2019semantic} is a large-scale scene parsing dataset with 150 foreground classes, covering a broad spectrum of objects and environments.

\noindent\textbf{Metrics.}
For evaluating segmentation performance, we adopt the mean intersection over union (mIoU) as our key metric. Since our approach does not directly provide semantic labels, we rely on the Hungarian matching algorithm ~\cite{kuhn1955hungarian} to establish an optimal correspondence between predicted masks and ground truth masks. This technique ensures precise alignment despite the absence of explicit label assignments. When working with datasets with a background class, we apply a many-to-one matching approach, enabling multiple predicted masks to link to a single background label effectively.

\subsection{Main Results}

\quad As shown in \autoref{tab:results}, we evaluate Falcon on six widely recognized datasets: Pascal VOC~\cite{everingham2015pascal}, Pascal Context~\cite{mottaghi2014role}, COCO-Object~\cite{lin2014microsoft}, COCO-Stuff-27, Cityscapes~\cite{cordts2016cityscapes}, and ADE20K~\cite{everingham2015pascal}. As shown in Table~\ref{tab:results}, Falcon achieves SOTA performance, surpassing the strongest baseline, DiffCut~\cite{couairon2025diffcut}, by +3.5 mIoU on COCO-Stuff-27, +4.3 on Cityscapes, +2.8 on ADE20K, and +1.4 on Pascal VOC. Compared to earlier approaches like MaskCut~\cite{wang2023cut} and DiffSeg~\cite{tian2024diffuse}, Falcon further widens the performance gap, presenting its robustness across diverse segmentation tasks.

\begin{figure}[t]
    \centering
    \includegraphics[width=0.35\textwidth]{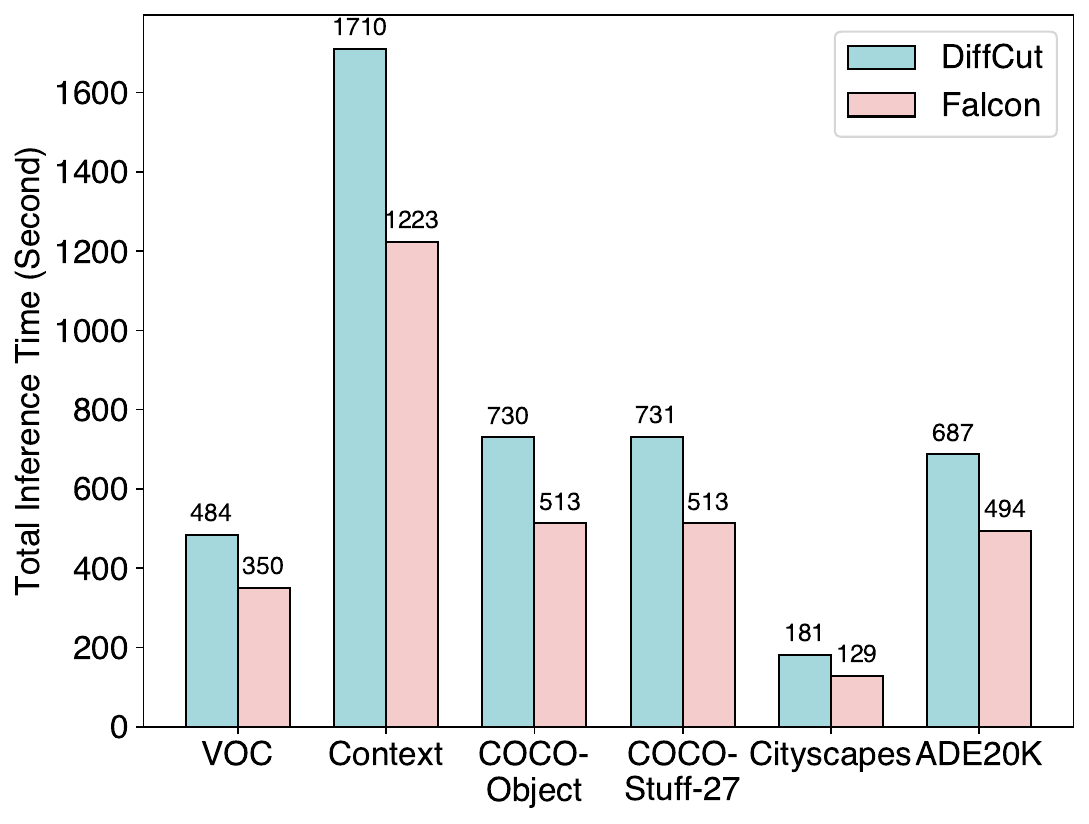}
    \vspace{-3mm}\caption{\textbf{The comparison of total evaluation time on various datasets.}  Falcon can shorten about 30\% inference time than recursive N-Cut on a single RTX4090.}
    \label{fig: eval_time}
    \vspace{-4mm}
\end{figure}

Beyond accuracy, as shown in \autoref{fig: eval_time}, Falcon significantly improves efficiency. By replacing recursive cuts in spectral clustering with parallel K-way optimization and refining segmentation through power-transformed affinity matrices and multi-modal mask refinement, Falcon achieves a 30\% faster than spectral clustering baselines on a single RTX4090 GPU. This combination of precision and speed makes Falcon well-suited for real-world applications, setting a new benchmark in unsupervised segmentation.

\subsection{Ablation Studies}

\quad In this part, we conduct a series of ablation studies to evaluate the individual and cumulative contributions of the key components in our Falcon framework: Power Transformation, Soft Assignment, Dynamic Affinity Matrix Regularization (DAMR), and Mask Refinement.

\noindent\textbf{Power Transformation in Affinity Matrix.}
\label{ssec:power_trans}
In high-dimensional embedding spaces, graph-based methods often struggle with \emph{near-uniform} affinity matrices, wherein even weakly related points exhibit artificially elevated similarity scores. This phenomenon, commonly referred to as \emph{similarity collapse}, masks genuine cluster boundaries and destabilizes spectral partitioning. By re-scaling pairwise similarities nonlinearly, the power transformation technique selectively amplifies stronger affinities while further suppressing weaker ones, thereby \emph{magnifying} the separation between distinct clusters. This enhanced contrast clarifies the spectral embedding and mitigates noise sensitivity, ultimately yielding more stable partitioning.
As in \autoref{fig: alpha vs miou}, our experiments on the ADE20K~\cite{everingham2015pascal} dataset show that tuning the power parameter $\alpha$ leads to marked improvements in mean Intersection-over-Union (mIoU), reflecting the practical benefits of this transformation in semantic segmentation. This aligns with the theoretical perspective that better-defined affinities foster sharper spectral distinctions, enabling leading eigenvectors to capture the intrinsic structure of the data more faithfully.

\noindent\textbf{Effects of Soft Assignment and Dynamic Affinity Matrix Regularization.} \autoref{tab:ablation_miou_cumulative} demonstrates that Soft Assignment and Dynamic Affinity Matrix Regularization markedly enhance Falcon's segmentation efficacy. Soft Assignment reformulates optimization to enable parallel cluster processing, reducing local sub-optimality and improving complex scene segmentation. Additionally, dynamic affinity matrix regularization refines the affinity structure to mitigate similarity collapse, thereby enhancing performance on complex datasets. Combined, these strategies overcome the inherent limitations of graph-cut methods, enabling globally coherent segmentations.

\noindent\textbf{Effects of Different Mask Refinements.}
We further investigate the impact of mask refinement within the Falcon pipeline by comparing three strategies: PAMR~\cite{Araslanov:2020:SSS}, DREAM (RGB), and DREAM (RGBD). As shown in \autoref{tab:ablation_miou_cumulative}, each method is applied after Soft Assignment and Dynamic Affinity Regularization, leveraging low-level cues (e.g., RGB or depth) to enhance coarse patch-level segmentations into higher-fidelity pixel-level masks. Across the benchmark datasets (except COCO-Stuff), all three refinements consistently improve mIoU over the pipeline without refinement. In particular, \textbf{PAMR}\cite{Araslanov:2020:SSS}demonstrates appreciable gains by sharpening object boundaries, while \textbf{DREAM (RGB)} refines segmentation with advanced RGB-based cues, outperforming PAMR on most datasets. Finally, \textbf{DREAM (RGBD)} achieves the strongest overall performance by incorporating both RGB and depth information, excelling in delineating complex shapes and intricate scene details. However, on COCO-Stuff, these refinements can inadvertently degrade performance, potentially due to extensive homogeneous background regions, numerous visually similar background categories, and complex object-background boundaries. These findings underscore the benefits of multi-modal integration for unsupervised segmentation, with DREAM (RGBD) delivering the most pronounced performance gains among all refinement strategies.
\vspace{-3mm}

\begin{figure}[t]
    \centering
    \includegraphics[width=0.4\textwidth]{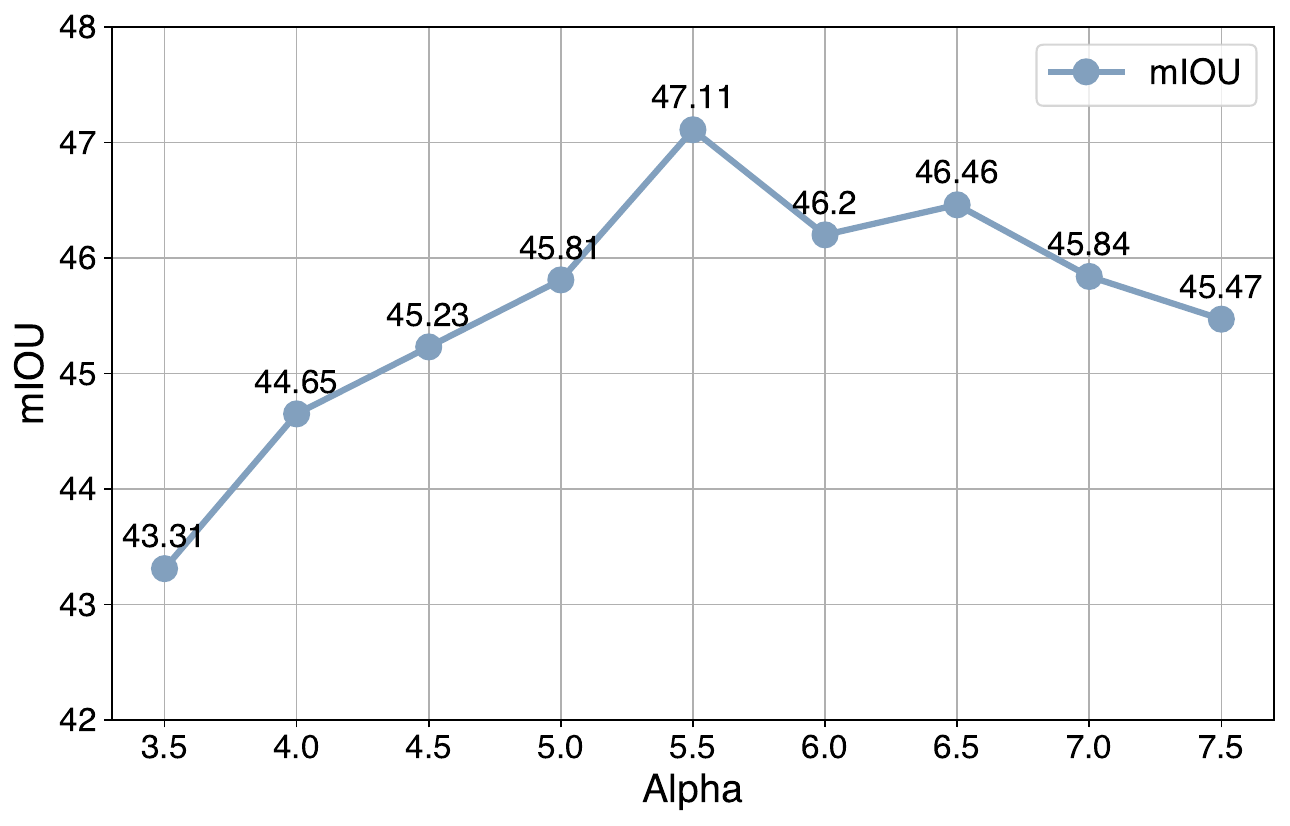}
    \vspace{-3mm}\caption{\textbf{Performance on ADE20k benchmark with different power transformation values of $\alpha$}. The mIoU achieves $47.1$ when $\alpha=5.5$.}
    \label{fig: alpha vs miou}
    \vspace{-5mm}
\end{figure}

\section{Conclusion}
\label{sec:conclusion}
\vspace{-1mm}
In this paper, we tackle limitations of recursive  N-Cut methods in graph-based unsupervised image segmentation, including suboptimal greedy partitions, susceptibility to local minima, and high computational costs. We introduce Falcon, a novel framework that leverages a regularized, parallel K-way Normalized Cut formulation, followed by refinement using low-level features. Experiments on benchmarks show that Falcon achieves state-of-the-art segmentation performance while reducing runtime by about 30\% compared to recursive graph cut baselines. These advancements underscore Falcon's scalability and efficiency, positioning it as a practical solution for real-world applications like image editing, autonomous driving, and medical imaging, without reliance on manual annotations. Future research will explore integrating multi-modal data and scaling to larger segmentation tasks to further enhance its capabilities.

\clearpage

\clearpage

{
    \small
    \bibliographystyle{ieeenat_fullname}
    \bibliography{main}

\begin{thebibliography}{78}
\providecommand{\natexlab}[1]{#1}
\providecommand{\url}[1]{\texttt{#1}}
\expandafter\ifx\csname urlstyle\endcsname\relax
  \providecommand{\doi}[1]{doi: #1}\else
  \providecommand{\doi}{doi: \begingroup \urlstyle{rm}\Url}\fi

\bibitem[Araslanov and Roth(2020)]{Araslanov:2020:SSS}
Nikita Araslanov and Stefan Roth.
\newblock Single-stage semantic segmentation from image labels.
\newblock In \emph{IEEE/CVF Conference on Computer Vision and Pattern Recognition (CVPR)}, pages 4253--4262, 2020.

\bibitem[Bellman and Dreyfus(2015)]{bellman2015applied}
Richard~E Bellman and Stuart~E Dreyfus.
\newblock \emph{Applied dynamic programming}.
\newblock Princeton university press, 2015.

\bibitem[Beyer et~al.(1999)Beyer, Goldstein, Ramakrishnan, and Shaft]{beyer1999nearest}
Kevin Beyer, Jonathan Goldstein, Raghu Ramakrishnan, and Uri Shaft.
\newblock When is “nearest neighbor” meaningful?
\newblock In \emph{Database Theory—ICDT’99: 7th International Conference Jerusalem, Israel, January 10--12, 1999 Proceedings 7}, pages 217--235. Springer, 1999.

\bibitem[Bochkovskii et~al.(2024)Bochkovskii, Delaunoy, Germain, Santos, Zhou, Richter, and Koltun]{bochkovskii2024depth}
Aleksei Bochkovskii, Ama{\~A}{\c{G}}l Delaunoy, Hugo Germain, Marcel Santos, Yichao Zhou, Stephan~R Richter, and Vladlen Koltun.
\newblock Depth pro: Sharp monocular metric depth in less than a second.
\newblock \emph{arXiv preprint arXiv:2410.02073}, 2024.

\bibitem[Caron et~al.(2021)Caron, Touvron, Misra, J{\'e}gou, Mairal, Bojanowski, and Joulin]{caron2021emerging}
Mathilde Caron, Hugo Touvron, Ishan Misra, Herv{\'e} J{\'e}gou, Julien Mairal, Piotr Bojanowski, and Armand Joulin.
\newblock Emerging properties in self-supervised vision transformers.
\newblock In \emph{Proceedings of the IEEE/CVF international conference on computer vision}, pages 9650--9660, 2021.

\bibitem[Cha et~al.(2023)Cha, Mun, and Roh]{cha2023learning}
Junbum Cha, Jonghwan Mun, and Byungseok Roh.
\newblock Learning to generate text-grounded mask for open-world semantic segmentation from only image-text pairs.
\newblock In \emph{Proceedings of the IEEE/CVF Conference on Computer Vision and Pattern Recognition}, pages 11165--11174, 2023.

\bibitem[Chen et~al.(2017)Chen, Papandreou, Kokkinos, Murphy, and Yuille]{chen2017deeplab}
Liang-Chieh Chen, George Papandreou, Iasonas Kokkinos, Kevin Murphy, and Alan~L Yuille.
\newblock Deeplab: Semantic image segmentation with deep convolutional nets, atrous convolution, and fully connected crfs.
\newblock \emph{IEEE transactions on pattern analysis and machine intelligence}, 40\penalty0 (4):\penalty0 834--848, 2017.

\bibitem[Chen et~al.(2022)Chen, Xiao, Nie, and Huang]{9706351}
Xiaojun Chen, Zhicong Xiao, Feiping Nie, and Joshua~Zhexue Huang.
\newblock Finc: An efficient and effective optimization method for normalized cut.
\newblock \emph{IEEE Transactions on Pattern Analysis and Machine Intelligence}, pages 1--1, 2022.

\bibitem[Chen et~al.(2024)Chen, Huang, Zhao, and Shen]{NEURIPS2024_a3017a8d}
Yannan Chen, Beichen Huang, Licheng Zhao, and Kaiming Shen.
\newblock Multidimensional fractional programming for normalized cuts.
\newblock In \emph{Advances in Neural Information Processing Systems}, pages 89563--89583. Curran Associates, Inc., 2024.

\bibitem[Cheng et~al.(2021)Cheng, Schwing, and Kirillov]{cheng2021per}
Bowen Cheng, Alex Schwing, and Alexander Kirillov.
\newblock Per-pixel classification is not all you need for semantic segmentation.
\newblock \emph{Advances in neural information processing systems}, 34:\penalty0 17864--17875, 2021.

\bibitem[Cheng et~al.(2022)Cheng, Misra, Schwing, Kirillov, and Girdhar]{cheng2022masked}
Bowen Cheng, Ishan Misra, Alexander~G Schwing, Alexander Kirillov, and Rohit Girdhar.
\newblock Masked-attention mask transformer for universal image segmentation.
\newblock In \emph{Proceedings of the IEEE/CVF conference on computer vision and pattern recognition}, pages 1290--1299, 2022.

\bibitem[Cho et~al.(2021)Cho, Mall, Bala, and Hariharan]{cho2021picie}
Jang~Hyun Cho, Utkarsh Mall, Kavita Bala, and Bharath Hariharan.
\newblock Picie: Unsupervised semantic segmentation using invariance and equivariance in clustering.
\newblock In \emph{Proceedings of the IEEE/CVF conference on computer vision and pattern recognition}, pages 16794--16804, 2021.

\bibitem[Cordts et~al.(2016)Cordts, Omran, Ramos, Rehfeld, Enzweiler, Benenson, Franke, Roth, and Schiele]{cordts2016cityscapes}
Marius Cordts, Mohamed Omran, Sebastian Ramos, Timo Rehfeld, Markus Enzweiler, Rodrigo Benenson, Uwe Franke, Stefan Roth, and Bernt Schiele.
\newblock The cityscapes dataset for semantic urban scene understanding.
\newblock In \emph{Proceedings of the IEEE conference on computer vision and pattern recognition}, pages 3213--3223, 2016.

\bibitem[Couairon et~al.(2025)Couairon, Shukor, Haugeard, Cord, and Thome]{couairon2025diffcut}
Paul Couairon, Mustafa Shukor, Jean-Emmanuel Haugeard, Matthieu Cord, and Nicolas Thome.
\newblock Diffcut: Catalyzing zero-shot semantic segmentation with diffusion features and recursive normalized cut.
\newblock \emph{Advances in Neural Information Processing Systems}, 37:\penalty0 13548--13578, 2025.

\bibitem[Darcet et~al.(2023)Darcet, Oquab, Mairal, and Bojanowski]{darcet2023vitneedreg}
Timothée Darcet, Maxime Oquab, Julien Mairal, and Piotr Bojanowski.
\newblock Vision transformers need registers, 2023.

\bibitem[Dong et~al.(2023)Dong, Bao, Zheng, Zhang, Chen, Yang, Zeng, Zhang, Yuan, Chen, et~al.]{dong2023maskclip}
Xiaoyi Dong, Jianmin Bao, Yinglin Zheng, Ting Zhang, Dongdong Chen, Hao Yang, Ming Zeng, Weiming Zhang, Lu Yuan, Dong Chen, et~al.
\newblock Maskclip: Masked self-distillation advances contrastive language-image pretraining.
\newblock In \emph{Proceedings of the IEEE/CVF Conference on Computer Vision and Pattern Recognition}, pages 10995--11005, 2023.

\bibitem[Donoho et~al.(2000)]{donoho2000high}
David~L Donoho et~al.
\newblock High-dimensional data analysis: The curses and blessings of dimensionality.
\newblock \emph{AMS math challenges lecture}, 1\penalty0 (2000):\penalty0 32, 2000.

\bibitem[Dosovitskiy et~al.(2020)Dosovitskiy, Beyer, Kolesnikov, Weissenborn, Zhai, Unterthiner, Dehghani, Minderer, Heigold, Gelly, et~al.]{dosovitskiy2020image}
Alexey Dosovitskiy, Lucas Beyer, Alexander Kolesnikov, Dirk Weissenborn, Xiaohua Zhai, Thomas Unterthiner, Mostafa Dehghani, Matthias Minderer, Georg Heigold, Sylvain Gelly, et~al.
\newblock An image is worth 16x16 words: Transformers for image recognition at scale.
\newblock \emph{arXiv preprint arXiv:2010.11929}, 2020.

\bibitem[Esser et~al.(2023)Esser, Chiu, Atighehchian, Granskog, and Germanidis]{esser2023structure}
Patrick Esser, Johnathan Chiu, Parmida Atighehchian, Jonathan Granskog, and Anastasis Germanidis.
\newblock Structure and content-guided video synthesis with diffusion models.
\newblock In \emph{Proceedings of the IEEE/CVF international conference on computer vision}, pages 7346--7356, 2023.

\bibitem[Everingham et~al.(2015)Everingham, Eslami, Van~Gool, Williams, Winn, and Zisserman]{everingham2015pascal}
Mark Everingham, SM~Ali Eslami, Luc Van~Gool, Christopher~KI Williams, John Winn, and Andrew Zisserman.
\newblock The pascal visual object classes challenge: A retrospective.
\newblock \emph{International journal of computer vision}, 111:\penalty0 98--136, 2015.

\bibitem[F and P(2004)]{efficient2004}
Felzenszwalb~P F and Huttenlocher~D P.
\newblock Efficient graph-based image segmentation.
\newblock \emph{International Journal of Computer Vision}, 2004.

\bibitem[Feng et~al.(2023)Feng, Gadde, Liao, Ramon, and Martinez]{feng2023network}
Qianli Feng, Raghudeep Gadde, Wentong Liao, Eduard Ramon, and Aleix Martinez.
\newblock Network-free, unsupervised semantic segmentation with synthetic images.
\newblock In \emph{Proceedings of the IEEE/CVF Conference on Computer Vision and Pattern Recognition}, pages 23602--23610, 2023.

\bibitem[Goyal et~al.(2021)Goyal, Caron, Lefaudeux, Xu, Wang, Pai, Singh, Liptchinsky, Misra, Joulin, et~al.]{goyal2021self}
Priya Goyal, Mathilde Caron, Benjamin Lefaudeux, Min Xu, Pengchao Wang, Vivek Pai, Mannat Singh, Vitaliy Liptchinsky, Ishan Misra, Armand Joulin, et~al.
\newblock Self-supervised pretraining of visual features in the wild.
\newblock \emph{arXiv preprint arXiv:2103.01988}, 2021.

\bibitem[Grady(2006)]{ranwalk2006}
L. Grady.
\newblock Random walks for image segmentation.
\newblock \emph{IEEE Transactions on Pattern Analysis and Machine Intelligence}, 2006.

\bibitem[Grill et~al.(2020)Grill, Strub, Altch{\'e}, Tallec, Richemond, Buchatskaya, Doersch, Avila~Pires, Guo, Gheshlaghi~Azar, et~al.]{grill2020bootstrap}
Jean-Bastien Grill, Florian Strub, Florent Altch{\'e}, Corentin Tallec, Pierre Richemond, Elena Buchatskaya, Carl Doersch, Bernardo Avila~Pires, Zhaohan Guo, Mohammad Gheshlaghi~Azar, et~al.
\newblock Bootstrap your own latent-a new approach to self-supervised learning.
\newblock \emph{Advances in neural information processing systems}, 33:\penalty0 21271--21284, 2020.

\bibitem[Gupta et~al.(2024)Gupta, Jaddipal, Prabhala, Paul, and Von~Platen]{gupta2024progressive}
Yatharth Gupta, Vishnu~V Jaddipal, Harish Prabhala, Sayak Paul, and Patrick Von~Platen.
\newblock Progressive knowledge distillation of stable diffusion xl using layer level loss.
\newblock \emph{arXiv preprint arXiv:2401.02677}, 2024.

\bibitem[Hamilton et~al.(2022)Hamilton, Zhang, Hariharan, Snavely, and Freeman]{hamilton2022unsupervised}
Mark Hamilton, Zhoutong Zhang, Bharath Hariharan, Noah Snavely, and William~T Freeman.
\newblock Unsupervised semantic segmentation by distilling feature correspondences.
\newblock \emph{arXiv preprint arXiv:2203.08414}, 2022.

\bibitem[Han et~al.(2024)Han, Jia, Li, Wang, Ivanovic, You, Liu, Wang, Pavone, Feng, et~al.]{han2024extrapolated}
Xiangyu Han, Zhen Jia, Boyi Li, Yan Wang, Boris Ivanovic, Yurong You, Lingjie Liu, Yue Wang, Marco Pavone, Chen Feng, et~al.
\newblock Extrapolated urban view synthesis benchmark.
\newblock \emph{arXiv preprint arXiv:2412.05256}, 2024.

\bibitem[He et~al.(2017)He, Gkioxari, Doll{\'a}r, and Girshick]{he2017mask}
Kaiming He, Georgia Gkioxari, Piotr Doll{\'a}r, and Ross Girshick.
\newblock Mask r-cnn.
\newblock In \emph{Proceedings of the IEEE international conference on computer vision}, pages 2961--2969, 2017.

\bibitem[He et~al.(2020)He, Fan, Wu, Xie, and Girshick]{he2020momentum}
Kaiming He, Haoqi Fan, Yuxin Wu, Saining Xie, and Ross Girshick.
\newblock Momentum contrast for unsupervised visual representation learning.
\newblock In \emph{Proceedings of the IEEE/CVF conference on computer vision and pattern recognition}, pages 9729--9738, 2020.

\bibitem[He et~al.(2022)He, Chen, Xie, Li, Doll{\'a}r, and Girshick]{he2022masked}
Kaiming He, Xinlei Chen, Saining Xie, Yanghao Li, Piotr Doll{\'a}r, and Ross Girshick.
\newblock Masked autoencoders are scalable vision learners.
\newblock In \emph{Proceedings of the IEEE/CVF conference on computer vision and pattern recognition}, pages 16000--16009, 2022.

\bibitem[Hu et~al.(2018)Hu, Shen, and Sun]{hu2018squeeze}
Jie Hu, Li Shen, and Gang Sun.
\newblock Squeeze-and-excitation networks.
\newblock In \emph{Proceedings of the IEEE conference on computer vision and pattern recognition}, pages 7132--7141, 2018.

\bibitem[Ji et~al.(2019)Ji, Henriques, and Vedaldi]{ji2019invariant}
Xu Ji, Joao~F Henriques, and Andrea Vedaldi.
\newblock Invariant information clustering for unsupervised image classification and segmentation.
\newblock In \emph{Proceedings of the IEEE/CVF international conference on computer vision}, pages 9865--9874, 2019.

\bibitem[Kuhn(1955)]{kuhn1955hungarian}
Harold~W Kuhn.
\newblock The hungarian method for the assignment problem.
\newblock \emph{Naval research logistics quarterly}, 2\penalty0 (1-2):\penalty0 83--97, 1955.

\bibitem[Langfelder and Horvath(2008)]{langfelder2008wgcna}
Peter Langfelder and Steve Horvath.
\newblock Wgcna: an r package for weighted correlation network analysis.
\newblock \emph{BMC bioinformatics}, 9:\penalty0 1--13, 2008.

\bibitem[Ledoux(2001)]{ledoux2001concentration}
Michel Ledoux.
\newblock \emph{The concentration of measure phenomenon}.
\newblock Number~89. American Mathematical Soc., 2001.

\bibitem[Li et~al.(2022{\natexlab{a}})Li, Li, Xiong, and Hoi]{li2022blip}
Junnan Li, Dongxu Li, Caiming Xiong, and Steven Hoi.
\newblock Blip: Bootstrapping language-image pre-training for unified vision-language understanding and generation.
\newblock In \emph{International conference on machine learning}, pages 12888--12900. PMLR, 2022{\natexlab{a}}.

\bibitem[Li et~al.(2023)Li, Wang, Cheng, Yu, Zhao, Song, Liu, Yuan, and Chen]{li2023acseg}
Kehan Li, Zhennan Wang, Zesen Cheng, Runyi Yu, Yian Zhao, Guoli Song, Chang Liu, Li Yuan, and Jie Chen.
\newblock Acseg: Adaptive conceptualization for unsupervised semantic segmentation.
\newblock In \emph{Proceedings of the IEEE/CVF conference on computer vision and pattern recognition}, pages 7162--7172, 2023.

\bibitem[Li et~al.(2022{\natexlab{b}})Li, Wang, Wang, and Zhao]{li2022hdmapnet}
Qi Li, Yue Wang, Yilun Wang, and Hang Zhao.
\newblock Hdmapnet: An online hd map construction and evaluation framework.
\newblock In \emph{2022 International Conference on Robotics and Automation (ICRA)}, pages 4628--4634. IEEE, 2022{\natexlab{b}}.

\bibitem[Liang et~al.(2023)Liang, Wu, Dai, Li, Zhao, Zhang, Zhang, Vajda, and Marculescu]{liang2023open}
Feng Liang, Bichen Wu, Xiaoliang Dai, Kunpeng Li, Yinan Zhao, Hang Zhang, Peizhao Zhang, Peter Vajda, and Diana Marculescu.
\newblock Open-vocabulary semantic segmentation with mask-adapted clip.
\newblock In \emph{Proceedings of the IEEE/CVF conference on computer vision and pattern recognition}, pages 7061--7070, 2023.

\bibitem[Lin et~al.(2014)Lin, Maire, Belongie, Hays, Perona, Ramanan, Doll{\'a}r, and Zitnick]{lin2014microsoft}
Tsung-Yi Lin, Michael Maire, Serge Belongie, James Hays, Pietro Perona, Deva Ramanan, Piotr Doll{\'a}r, and C~Lawrence Zitnick.
\newblock Microsoft coco: Common objects in context.
\newblock In \emph{Computer vision--ECCV 2014: 13th European conference, zurich, Switzerland, September 6-12, 2014, proceedings, part v 13}, pages 740--755. Springer, 2014.

\bibitem[Long et~al.(2015)Long, Shelhamer, and Darrell]{long2015fully}
Jonathan Long, Evan Shelhamer, and Trevor Darrell.
\newblock Fully convolutional networks for semantic segmentation.
\newblock In \emph{Proceedings of the IEEE conference on computer vision and pattern recognition}, pages 3431--3440, 2015.

\bibitem[Mottaghi et~al.(2014)Mottaghi, Chen, Liu, Cho, Lee, Fidler, Urtasun, and Yuille]{mottaghi2014role}
Roozbeh Mottaghi, Xianjie Chen, Xiaobai Liu, Nam-Gyu Cho, Seong-Whan Lee, Sanja Fidler, Raquel Urtasun, and Alan Yuille.
\newblock The role of context for object detection and semantic segmentation in the wild.
\newblock In \emph{Proceedings of the IEEE conference on computer vision and pattern recognition}, pages 891--898, 2014.

\bibitem[Namekata et~al.(2024)Namekata, Sabour, Fidler, and Kim]{namekata2024emerdiff}
Koichi Namekata, Amirmojtaba Sabour, Sanja Fidler, and Seung~Wook Kim.
\newblock Emerdiff: Emerging pixel-level semantic knowledge in diffusion models.
\newblock \emph{arXiv preprint arXiv:2401.11739}, 2024.

\bibitem[Ng et~al.(2001)Ng, Jordan, and Weiss]{ng2001spectral}
Andrew Ng, Michael Jordan, and Yair Weiss.
\newblock On spectral clustering: Analysis and an algorithm.
\newblock \emph{Advances in neural information processing systems}, 14, 2001.

\bibitem[Oquab et~al.(2023)Oquab, Darcet, Moutakanni, Vo, Szafraniec, Khalidov, Fernandez, Haziza, Massa, El-Nouby, et~al.]{oquab2023dinov2}
Maxime Oquab, Timoth{\'e}e Darcet, Th{\'e}o Moutakanni, Huy Vo, Marc Szafraniec, Vasil Khalidov, Pierre Fernandez, Daniel Haziza, Francisco Massa, Alaaeldin El-Nouby, et~al.
\newblock Dinov2: Learning robust visual features without supervision.
\newblock \emph{arXiv preprint arXiv:2304.07193}, 2023.

\bibitem[Radford et~al.(2021)Radford, Kim, Hallacy, Ramesh, Goh, Agarwal, Sastry, Askell, Mishkin, Clark, et~al.]{radford2021learning}
Alec Radford, Jong~Wook Kim, Chris Hallacy, Aditya Ramesh, Gabriel Goh, Sandhini Agarwal, Girish Sastry, Amanda Askell, Pamela Mishkin, Jack Clark, et~al.
\newblock Learning transferable visual models from natural language supervision.
\newblock In \emph{International conference on machine learning}, pages 8748--8763. PmLR, 2021.

\bibitem[Radovanovic et~al.(2010)Radovanovic, Nanopoulos, and Ivanovic]{radovanovic2010hubs}
Milos Radovanovic, Alexandros Nanopoulos, and Mirjana Ivanovic.
\newblock Hubs in space: Popular nearest neighbors in high-dimensional data.
\newblock \emph{Journal of Machine Learning Research}, 11\penalty0 (sept):\penalty0 2487--2531, 2010.

\bibitem[Ranasinghe et~al.(2023)Ranasinghe, McKinzie, Ravi, Yang, Toshev, and Shlens]{ranasinghe2023perceptual}
Kanchana Ranasinghe, Brandon McKinzie, Sachin Ravi, Yinfei Yang, Alexander Toshev, and Jonathon Shlens.
\newblock Perceptual grouping in contrastive vision-language models.
\newblock In \emph{Proceedings of the IEEE/CVF International Conference on Computer Vision}, pages 5571--5584, 2023.

\bibitem[Ravi et~al.(2024)Ravi, Gabeur, Hu, Hu, Ryali, Ma, Khedr, R{\"a}dle, Rolland, Gustafson, Mintun, Pan, Alwala, Carion, Wu, Girshick, Doll{\'a}r, and Feichtenhofer]{ravi2024sam2}
Nikhila Ravi, Valentin Gabeur, Yuan-Ting Hu, Ronghang Hu, Chaitanya Ryali, Tengyu Ma, Haitham Khedr, Roman R{\"a}dle, Chloe Rolland, Laura Gustafson, Eric Mintun, Junting Pan, Kalyan~Vasudev Alwala, Nicolas Carion, Chao-Yuan Wu, Ross Girshick, Piotr Doll{\'a}r, and Christoph Feichtenhofer.
\newblock Sam 2: Segment anything in images and videos.
\newblock \emph{arXiv preprint arXiv:2408.00714}, 2024.

\bibitem[Ren et~al.(2023)Ren, Li, Xu, Zhu, Wang, Liu, Chang, and Liang]{ren2023viewco}
Pengzhen Ren, Changlin Li, Hang Xu, Yi Zhu, Guangrun Wang, Jianzhuang Liu, Xiaojun Chang, and Xiaodan Liang.
\newblock Viewco: Discovering text-supervised segmentation masks via multi-view semantic consistency.
\newblock \emph{arXiv preprint arXiv:2302.10307}, 2023.

\bibitem[Sapkota et~al.(2024)Sapkota, Ahmed, and Karkee]{sapkota2024comparing}
Ranjan Sapkota, Dawood Ahmed, and Manoj Karkee.
\newblock Comparing yolov8 and mask r-cnn for instance segmentation in complex orchard environments.
\newblock \emph{Artificial Intelligence in Agriculture}, 13:\penalty0 84--99, 2024.

\bibitem[Shi and Malik(2000{\natexlab{a}})]{ncut2000}
Jianbo Shi and Jitendra Malik.
\newblock Normalized cuts and image segmentation.
\newblock \emph{IEEE Transactions on Pattern Analysis and Machine Intelligence}, 2000{\natexlab{a}}.

\bibitem[Shi and Malik(2000{\natexlab{b}})]{shi2000normalized}
Jianbo Shi and Jitendra Malik.
\newblock Normalized cuts and image segmentation.
\newblock \emph{IEEE Transactions on pattern analysis and machine intelligence}, 22\penalty0 (8):\penalty0 888--905, 2000{\natexlab{b}}.

\bibitem[Sick et~al.(2024)Sick, Engel, Hermosilla, and Ropinski]{sick2024unsupervised}
Leon Sick, Dominik Engel, Pedro Hermosilla, and Timo Ropinski.
\newblock Unsupervised semantic segmentation through depth-guided feature correlation and sampling.
\newblock In \emph{Proceedings of the IEEE/CVF Conference on Computer Vision and Pattern Recognition}, pages 3637--3646, 2024.

\bibitem[Tian et~al.(2024)Tian, Aggarwal, Colaco, Kira, and Gonzalez-Franco]{tian2024diffuse}
Junjiao Tian, Lavisha Aggarwal, Andrea Colaco, Zsolt Kira, and Mar Gonzalez-Franco.
\newblock Diffuse attend and segment: Unsupervised zero-shot segmentation using stable diffusion.
\newblock In \emph{Proceedings of the IEEE/CVF Conference on Computer Vision and Pattern Recognition}, pages 3554--3563, 2024.

\bibitem[Tschannen et~al.(2025)Tschannen, Gritsenko, Wang, Naeem, Alabdulmohsin, Parthasarathy, Evans, Beyer, Xia, Mustafa, et~al.]{tschannen2025siglip}
Michael Tschannen, Alexey Gritsenko, Xiao Wang, Muhammad~Ferjad Naeem, Ibrahim Alabdulmohsin, Nikhil Parthasarathy, Talfan Evans, Lucas Beyer, Ye Xia, Basil Mustafa, et~al.
\newblock Siglip 2: Multilingual vision-language encoders with improved semantic understanding, localization, and dense features.
\newblock \emph{arXiv preprint arXiv:2502.14786}, 2025.

\bibitem[V(2004)]{specclus2007}
Luxburg~U V.
\newblock A tutorial on spectral clustering.
\newblock \emph{Statistics and Computing}, 2004.

\bibitem[Vershynin(2018)]{vershynin2018high}
Roman Vershynin.
\newblock \emph{High-dimensional probability: An introduction with applications in data science}.
\newblock Cambridge university press, 2018.

\bibitem[Wang et~al.(2021)Wang, Zhu, Adam, Yuille, and Chen]{wang2021max}
Huiyu Wang, Yukun Zhu, Hartwig Adam, Alan Yuille, and Liang-Chieh Chen.
\newblock Max-deeplab: End-to-end panoptic segmentation with mask transformers.
\newblock In \emph{Proceedings of the IEEE/CVF conference on computer vision and pattern recognition}, pages 5463--5474, 2021.

\bibitem[Wang and Yang(2021)]{wang2021remote}
S Wang and F Yang.
\newblock Remote sensing image semantic segmentation method based on u-net feature fusion optimization strategy.
\newblock \emph{Comput. Sci}, 48\penalty0 (8):\penalty0 162--168, 2021.

\bibitem[Wang et~al.(2023{\natexlab{a}})Wang, Girdhar, Yu, and Misra]{wang2023cut}
Xudong Wang, Rohit Girdhar, Stella~X Yu, and Ishan Misra.
\newblock Cut and learn for unsupervised object detection and instance segmentation.
\newblock In \emph{Proceedings of the IEEE/CVF conference on computer vision and pattern recognition}, pages 3124--3134, 2023{\natexlab{a}}.

\bibitem[Wang et~al.(2022)Wang, Shen, Hu, Yuan, Crowley, and Vaufreydaz]{tokencut2022}
Yangtao Wang, Xi Shen, Shell~Xu Hu, Yuan Yuan, James~L. Crowley, and Dominique Vaufreydaz.
\newblock Self-supervised transformers for unsupervised object discovery using normalized cut.
\newblock In \emph{Proceedings of the IEEE/CVF Conference on Computer Vision and Pattern Recognition (CVPR)}, pages 14543--14553, 2022.

\bibitem[Wang et~al.(2023{\natexlab{b}})Wang, Shen, Yuan, Du, Li, Hu, Crowley, and Vaufreydaz]{wang2023tokencut}
Yangtao Wang, Xi Shen, Yuan Yuan, Yuming Du, Maomao Li, Shell~Xu Hu, James~L Crowley, and Dominique Vaufreydaz.
\newblock Tokencut: Segmenting objects in images and videos with self-supervised transformer and normalized cut.
\newblock \emph{IEEE transactions on pattern analysis and machine intelligence}, 45\penalty0 (12):\penalty0 15790--15801, 2023{\natexlab{b}}.

\bibitem[Wen et~al.(2022)Wen, Zhao, Zheng, Zhang, and Qi]{wen2022self}
Xin Wen, Bingchen Zhao, Anlin Zheng, Xiangyu Zhang, and Xiaojuan Qi.
\newblock Self-supervised visual representation learning with semantic grouping.
\newblock \emph{Advances in neural information processing systems}, 35:\penalty0 16423--16438, 2022.

\bibitem[Xie et~al.(2021)Xie, Wang, Yu, Anandkumar, Alvarez, and Luo]{xie2021segformer}
Enze Xie, Wenhai Wang, Zhiding Yu, Anima Anandkumar, Jose~M Alvarez, and Ping Luo.
\newblock Segformer: Simple and efficient design for semantic segmentation with transformers.
\newblock \emph{Advances in neural information processing systems}, 34:\penalty0 12077--12090, 2021.

\bibitem[Xu et~al.(2022)Xu, De~Mello, Liu, Byeon, Breuel, Kautz, and Wang]{xu2022groupvit}
Jiarui Xu, Shalini De~Mello, Sifei Liu, Wonmin Byeon, Thomas Breuel, Jan Kautz, and Xiaolong Wang.
\newblock Groupvit: Semantic segmentation emerges from text supervision.
\newblock In \emph{Proceedings of the IEEE/CVF conference on computer vision and pattern recognition}, pages 18134--18144, 2022.

\bibitem[Yuille and Rangarajan(2001)]{NIPS2001_a0128693}
Alan~L Yuille and Anand Rangarajan.
\newblock The concave-convex procedure (cccp).
\newblock In \emph{Advances in Neural Information Processing Systems}. MIT Press, 2001.

\bibitem[Zhai et~al.(2023)Zhai, Mustafa, Kolesnikov, and Beyer]{zhai2023sigmoid}
Xiaohua Zhai, Basil Mustafa, Alexander Kolesnikov, and Lucas Beyer.
\newblock Sigmoid loss for language image pre-training.
\newblock In \emph{Proceedings of the IEEE/CVF international conference on computer vision}, pages 11975--11986, 2023.

\bibitem[Zhang et~al.(2019)Zhang, Zhao, Qiao, Wang, and Li]{Zhang_2019_CVPR}
Xiao Zhang, Rui Zhao, Yu Qiao, Xiaogang Wang, and Hongsheng Li.
\newblock Adacos: Adaptively scaling cosine logits for effectively learning deep face representations.
\newblock In \emph{Proceedings of the IEEE/CVF Conference on Computer Vision and Pattern Recognition (CVPR)}, 2019.

\bibitem[Zhang et~al.(2020)Zhang, Chu, Leng, and Miao]{Zhang2020MaskRefinedRA}
Yiqing Zhang, Jun Chu, Lu Leng, and Jun Miao.
\newblock Mask-refined r-cnn: A network for refining object details in instance segmentation.
\newblock \emph{Sensors (Basel, Switzerland)}, 20, 2020.

\bibitem[Zhao et~al.(2017)Zhao, Shi, Qi, Wang, and Jia]{zhao2017pyramid}
Hengshuang Zhao, Jianping Shi, Xiaojuan Qi, Xiaogang Wang, and Jiaya Jia.
\newblock Pyramid scene parsing network.
\newblock In \emph{Proceedings of the IEEE conference on computer vision and pattern recognition}, pages 2881--2890, 2017.

\bibitem[Zhou et~al.(2019{\natexlab{a}})Zhou, Zhao, Puig, Xiao, Fidler, Barriuso, and Torralba]{zhou2019semantic}
Bolei Zhou, Hang Zhao, Xavier Puig, Tete Xiao, Sanja Fidler, Adela Barriuso, and Antonio Torralba.
\newblock Semantic understanding of scenes through the ade20k dataset.
\newblock \emph{International Journal of Computer Vision}, 127:\penalty0 302--321, 2019{\natexlab{a}}.

\bibitem[Zhou et~al.(2021)Zhou, Wei, Wang, Shen, Xie, Yuille, and Kong]{zhou2021ibot}
Jinghao Zhou, Chen Wei, Huiyu Wang, Wei Shen, Cihang Xie, Alan Yuille, and Tao Kong.
\newblock ibot: Image bert pre-training with online tokenizer.
\newblock \emph{arXiv preprint arXiv:2111.07832}, 2021.

\bibitem[Zhou et~al.(2019{\natexlab{b}})Zhou, Siddiquee, Tajbakhsh, and Liang]{zhou2019unet++}
Zongwei Zhou, Md~Mahfuzur~Rahman Siddiquee, Nima Tajbakhsh, and Jianming Liang.
\newblock Unet++: Redesigning skip connections to exploit multiscale features in image segmentation.
\newblock \emph{IEEE transactions on medical imaging}, 39\penalty0 (6):\penalty0 1856--1867, 2019{\natexlab{b}}.

\bibitem[Zhu et~al.(2016)Zhu, Meng, Cai, and Lu]{zhu2016beyond}
Hongyuan Zhu, Fanman Meng, Jianfei Cai, and Shijian Lu.
\newblock Beyond pixels: A comprehensive survey from bottom-up to semantic image segmentation and cosegmentation.
\newblock \emph{Journal of Visual Communication and Image Representation}, 34:\penalty0 12--27, 2016.

\bibitem[Ziegler and Asano(2022)]{ziegler2022self}
Adrian Ziegler and Yuki~M Asano.
\newblock Self-supervised learning of object parts for semantic segmentation.
\newblock In \emph{Proceedings of the IEEE/CVF conference on computer vision and pattern recognition}, pages 14502--14511, 2022.

\bibitem[Zimek et~al.(2012)Zimek, Schubert, and Kriegel]{zimek2012survey}
Arthur Zimek, Erich Schubert, and Hans-Peter Kriegel.
\newblock A survey on unsupervised outlier detection in high-dimensional numerical data.
\newblock \emph{Statistical Analysis and Data Mining: The ASA Data Science Journal}, 5\penalty0 (5):\penalty0 363--387, 2012.

\end{thebibliography}
}

\clearpage
\setcounter{page}{1}
\maketitlesupplementary

\section*{Appendix A: Normalized Cut Formulation and Spectral Relaxation}
\addcontentsline{toc}{section}{Appendix A: Normalized Cut Formulation and Spectral Relaxation}
\label{app:ncut}

\subsection*{Problem Statement}
Let \( \mathcal{G} = (\mathcal{V}, \mathcal{E}, \bm{W}) \) be a weighted undirected graph with \( |\mathcal{V}| = N \), adjacency matrix \( \bm{W} \in \mathbb{R}_+^{N \times N} \), and diagonal degree matrix \( \bm{D} = \mathrm{diag}(\bm{W}\mathbf{1}_N) \). The \emph{normalized cut} (\(\mathrm{Ncut}\)) objective seeks a partition of \( \mathcal{V} \) into \( K \) disjoint subsets \( \{\mathcal{A}_k\}_{k=1}^K \) that minimizes connectivity between clusters relative to their volumes. For a binary partition (\( K=2 \)), the objective is:
\[
\mathrm{Ncut}(S, \bar{S}) = \frac{\mathrm{cut}(S, \bar{S})}{\mathrm{vol}(S)} + \frac{\mathrm{cut}(\bar{S}, S)}{\mathrm{vol}(\bar{S})},
\]
where \( S \subset \mathcal{V} \), \( \mathrm{cut}(A,B) = \sum_{i \in A, j \in B} W_{ij} \), and \( \mathrm{vol}(A) = \sum_{i \in A} d_i \). For a \( K \)-way partition, the generalized form is:
\[
\operatorname{Ncut}(\{\mathcal{A}_k\}) = \sum_{k=1}^K \frac{\mathrm{cut}(\mathcal{A}_k, \mathcal{V}\setminus\mathcal{A}_k)}{\mathrm{vol}(\mathcal{A}_k)}.
\]

\subsection*{Discrete Formulation and Spectral Relaxation}
The discrete optimization problem is NP-hard. Let \( \bm{x}_k \in \{0,1\}^N \) be binary indicator vectors for clusters \( \{\mathcal{A}_k\} \). The objective can be rewritten as:
\[
\operatorname{Ncut}(\{\mathcal{A}_k\}) = K - \sum_{k=1}^K \frac{\bm{x}_k^\top \bm{W} \bm{x}_k}{\bm{x}_k^\top \bm{D} \bm{x}_k}.
\]
To relax this, replace \( \bm{x}_k \) with continuous vectors. For \( K=2 \), define an indicator \( \bm{f} \in \mathbb{R}^N \), constrained to \( \bm{f}_i \in \{\pm \alpha\} \) for discrete partitions. The relaxed problem becomes:
\[
\min_{\bm{f} \in \mathbb{R}^N} \frac{\bm{f}^\top \bm{L} \bm{f}}{\bm{f}^\top \bm{D} \bm{f}}, \quad \text{where } \bm{L} = \bm{D} - \bm{W},
\]
subject to \( \bm{f}^\top \bm{D} \mathbf{1} = 0 \). For multiple clusters (\( K>2 \)), introduce a matrix \( \bm{X} \in \mathbb{R}_+^{N \times K} \) with \( \bm{X}\mathbf{1}_K = \mathbf{1}_N \), leading to:
\[
\max_{\bm{X} \succeq 0} \sum_{k=1}^K \frac{\bm{x}_k^\top \bm{W} \bm{x}_k}{\bm{x}_k^\top \bm{D} \bm{x}_k}.
\]
Unlike standard spectral clustering, which enforces \( \bm{X}^\top \bm{D} \bm{X} = \bm{I}_K \), this formulation allows more flexible assignments.

\subsection*{Spectral Solution and Normalized Laplacians}
The relaxed problem reduces to finding eigenvectors of the Laplacian. For \( K=2 \), the solution is the second eigenvector of the \emph{symmetric normalized Laplacian}:
\[
\bm{L}_{\mathrm{sym}} = \bm{D}^{-1/2} \bm{L} \bm{D}^{-1/2},
\]
with the Rayleigh quotient:
\[
\min_{\tilde{\bm{f}} \perp \bm{D}^{1/2}\mathbf{1}} \frac{\tilde{\bm{f}}^\top \bm{L}_{\mathrm{sym}} \tilde{\bm{f}}}{\tilde{\bm{f}}^\top \tilde{\bm{f}}}, \quad \tilde{\bm{f}} = \bm{D}^{1/2} \bm{f}.
\]
For \( K>2 \), the first \( K \) eigenvectors of \( \bm{L}_{\mathrm{sym}} \) or \( \bm{D}^{-1}\bm{L} \) are used to form \( \bm{X} \), followed by clustering (e.g., k-means).

\subsection*{The Relaxation Gap and Practical Considerations}
The continuous solution may deviate from the ideal discrete partition due to:
\begin{itemize}
    \item \emph{Discrete vs.\ Continuous Feasibility}: Eigenvectors \( \bm{f} \in \mathbb{R}^N \) are not binary.
    \item \emph{Approximation Error}: Thresholding (e.g., by sign) introduces discrepancies.
    \item \emph{Global vs.\ Local Optimality}: The spectral solution is globally optimal in the relaxed space but suboptimal in the discrete space.
\end{itemize}

\textbf{Implications in Practice:}
\begin{itemize}
    \item \emph{Binary Splitting}: Thresholding the second eigenvector provides a heuristic partition.
    \item \emph{Multi-Way Clustering}: Using \( K \) eigenvectors with k-means introduces additional approximations.
    \item \emph{Refinement}: Post-processing (e.g., greedy optimization) can reduce the gap at higher computational cost.
\end{itemize}

\subsection*{Conclusion}
Spectral relaxation transforms the NP-hard normalized cut problem into a tractable eigenvalue problem. While the continuous solution is globally optimal, the \emph{relaxation gap}—the discrepancy between continuous and discrete optima—remains a fundamental limitation. Nevertheless, spectral methods strike an effective balance between computational efficiency and solution quality, making them indispensable for large-scale graph clustering.

\section*{Appendix B: Optimization Framework and Convergence Analysis}  
\addcontentsline{toc}{section}{Appendix B: Optimization Framework and Convergence Analysis}  
\label{app:framework}  

\subsection*{B.1 Fractional Quadratic Transform for Ratio Maximization}  
\label{app:frac_transform}  

The fractional quadratic transform (FQT) provides a mechanism to decouple ratio terms in optimization objectives. We restate the key lemma and its application to our problem:  

\begin{lemma}[Quadratic Transform for Ratios]  
\label{lem:quad_transform}  
For \( a > 0 \) and \( b > 0 \), the following equality holds:  
\[
\frac{a}{b} = \sup_{y \geq 0} \left( 2y\sqrt{a} - y^2 b \right).  
\]  
\end{lemma}  

\begin{proof}  
Define \( f(y) = 2y\sqrt{a} - y^2 b \). Differentiating with respect to \( y \):  
\[
f'(y) = 2\sqrt{a} - 2y b.  
\]  
Setting \( f'(y)=0 \) yields \( y^* = \sqrt{a}/b \). Substituting \( y^* \) into \( f(y) \):  
\[
f(y^*) = \frac{2a}{b} - \frac{a}{b} = \frac{a}{b}.  
\]  
Thus, the supremum is achieved at \( y^* \), verifying the identity.  
\end{proof}  

\textbf{Application to Ncut Objective:}  
For each cluster \( k \), define:  
\[
a_k = \bm{x}_k^\top \bm{W} \bm{x}_k, \quad b_k = \bm{x}_k^\top \bm{D} \bm{x}_k.  
\]  
By Lemma~\ref{lem:quad_transform}, the ratio \( \frac{a_k}{b_k} \) can be rewritten as:  
\[
\frac{\bm{x}_k^\top \bm{W} \bm{x}_k}{\bm{x}_k^\top \bm{D} \bm{x}_k} = \max_{y_k \geq 0} \left( 2y_k \sqrt{\bm{x}_k^\top \bm{W} \bm{x}_k} - y_k^2 \bm{x}_k^\top \bm{D} \bm{x}_k \right).  
\]  
Summing over all \( K \) clusters transforms the original Ncut maximization into:  
\[
\max_{\substack{\bm{X} \succeq 0, \\ \bm{y} \succeq 0}} \sum_{k=1}^K \left( 2y_k \sqrt{\bm{x}_k^\top \bm{W} \bm{x}_k} - y_k^2 \bm{x}_k^\top \bm{D} \bm{x}_k \right),  
\]  
where \( \bm{X}\1_K = \1_N \) is enforced to maintain partition constraints.  

\subsection*{B.2 Alternating Optimization and Convergence Guarantees}  
\label{app:convergence}  

\begin{assumption}  
\label{assump:convergence}  
\begin{enumerate}[label=(\roman*)]  
    \item The feasible set \( \mathcal{X} = \{\bm{X} \succeq 0 \mid \bm{X}\1_K = \1_N\} \) is compact.  
    \item Matrices \( \bm{W} \) and \( \bm{D} \) have finite entries, with \( \bm{D} \succ 0 \).  
\end{enumerate}  
\end{assumption}  

\begin{theorem}[Monotonic Convergence]  
\label{thm:convergence}  
Under Assumption~\ref{assump:convergence}, the alternating updates generate a sequence \( \{\mathcal{L}^{(t)}\} \) satisfying:  
\[
\mathcal{L}^{(t+1)} \geq \mathcal{L}^{(t)}, \quad \forall t \geq 0,  
\]  
with convergence to a stationary point of the objective.  
\end{theorem}  

\begin{proof}  
The proof follows from analyzing the two-phase alternating optimization procedure:  

\noindent \textbf{Phase 1 (Update \( \bm{y} \)):}  
For fixed \( \bm{X} \), the optimal auxiliary variables \( \bm{y} \) are computed via:  
\[
y_k^* = \sqrt{\frac{\bm{x}_k^\top \bm{W} \bm{x}_k}{\bm{x}_k^\top \bm{D} \bm{x}_k}},  
\]  
which globally maximizes each term in the FQT-transformed objective (Lemma~\ref{lem:quad_transform}). This guarantees:  
\[
\mathcal{L}(\bm{X}, \bm{y}^{(t+1)}) \geq \mathcal{L}(\bm{X}, \bm{y}^{(t)}).  
\]  

\noindent \textbf{Phase 2 (Update \( \bm{X} \)):}  
For fixed \( \bm{y} \), the soft assignment matrix \( \bm{X} \) is updated via a constrained mirror ascent step. Let \( \mathcal{L}(\bm{X}) \) denote the FQT objective. The gradient with respect to \( X_{ik} \) is:  
\[
\nabla_{X_{ik}} \mathcal{L} = \frac{2 y_k (\bm{W}\bm{x}_k)_i}{\bm{x}_k^\top \bm{D} \bm{x}_k} - \frac{2 y_k^2 D_{ii} X_{ik}}{\bm{x}_k^\top \bm{D} \bm{x}_k}.  
\]  
To maintain feasibility (\( \bm{X}_i \in \Delta^{K-1} \)), we solve:  
\[
\bm{X}_i^{\text{new}} = \arg\max_{\bm{X}_i \in \Delta^{K-1}} \left\langle \nabla_{\bm{X}_i} \mathcal{L}, \bm{X}_i \right\rangle - \frac{1}{\eta} D_{\text{KL}}(\bm{X}_i \| \bm{X}_i^{\text{old}}),  
\]  
where \( \eta > 0 \) is an implicit step size. The closed-form solution is derived as:  
\[
{X}_{ik}^{\text{new}} \propto {X}_{ik}^{\text{old}} \exp\left( \eta \cdot \frac{\sum_j {W}_{ij} {X}_{jk}^{\text{old}}}{\sum_j {X}_{jk}^{\text{old}} {D}_{jj}} y_k\right),  
\]  
which reduces to the softmax update rule after normalization. This step ensures:  
\[
\mathcal{L}(\bm{X}^{(t+1)}, \bm{y}) \geq \mathcal{L}(\bm{X}^{(t)}, \bm{y}).  
\]  

\noindent \textbf{Convergence to Stationarity:}  
The sequence \( \{\mathcal{L}^{(t)}\} \) is non-decreasing and bounded above due to:  
\begin{itemize}  
\item Compactness of \( \mathcal{X} \) (Assumption~\ref{assump:convergence}(i)),  
\item Boundedness of \( \bm{W} \) and \( \bm{D} \) (Assumption~\ref{assump:convergence}(ii)).  
\end{itemize}  
By the monotone convergence theorem, \( \{\mathcal{L}^{(t)}\} \) converges to a limit \( \mathcal{L}^* \). The smoothness of the objective and the update rules further ensure that \( \mathcal{L}^* \) corresponds to a stationary point.  
\end{proof}  


\section*{Appendix C: Graph Reweighting, Edge Sparsity, and Regularization Effects}  
\addcontentsline{toc}{section}{Appendix D: Graph Reweighting, Edge Sparsity, and Regularization Effects}  
\label{app:graph_update}  

\begin{definition}[Assignment-Consistent Reweighting]
\label{def:reweight}
The edge weight update rule is given by:
\[
{W}_{ij}^{\mathrm{new}} = {W}_{ij} \cdot \exp\left( -\frac{(1 - \cos_{ij})^2}{\beta} \right), \quad \] 
\[ \text{where} \quad \cos_{ij} = \frac{\langle \bm{X}_i, \bm{X}_j \rangle}{\|\bm{X}_i\|_2 \|\bm{X}_j\|_2}.
\]
\end{definition}

\begin{proposition}[Structural Invariance]  
\label{prop:invariance}  
The update rule preserves key graph properties:  
\begin{enumerate}[label=(\alph*)]  
    \item \textbf{Symmetry:} If \( \bm{W} = \bm{W}^\top \), then \( \bm{W}^{\mathrm{new}} = (\bm{W}^{\mathrm{new}})^\top \).  
    \item \textbf{Degree Adaptation:} The updated degree matrix \( \bm{D}^{\mathrm{new}} = \diag(\bm{W}^{\mathrm{new}} \1_N) \) reflects an adaptive renormalization of node connectivity.  
\end{enumerate}  
\end{proposition}  

\textbf{Edge Sparsity and Regularization.}  
The reweighting function  
\[
\exp\left( -\frac{(1 - \cos_{ij})^2}{\beta} \right)  
\]  
serves dual purposes: sparsification and implicit regularization.  

1. \textbf{Sparsification Effect:}  
When \( \beta \) is small, the exponential term rapidly decays edge weights for pairs \( (i,j) \) with low assignment similarity \( s_{ij} \). This suppresses weak or noisy connections, effectively sparsifying the graph and sharpening cluster boundaries. Conversely, large \( \beta \) preserves more edges, maintaining global connectivity at the cost of potential ambiguity.  

2. \textbf{Regularization Role:}  
The parameter \( \beta \) acts as a regularization knob:  
\begin{itemize}  
    \item \emph{Stability:} By controlling the rate of weight decay, \( \beta \) prevents abrupt changes in graph topology during iterative updates. This stabilizes the optimization trajectory, avoiding oscillations in cluster assignments.  
    \item \emph{Adaptive Smoothness:} The update rule smooths the graph structure by emphasizing edges aligned with current assignments while downweighting inconsistent ones. This adaptively enforces local consistency without enforcing rigid pairwise constraints.  
    \item \emph{Noise Suppression:} The exponential suppression of low-similarity edges inherently filters out transient or spurious connections, akin to a soft thresholding mechanism.  
\end{itemize}  

\textbf{Balancing Trade-offs:}  
While smaller \( \beta \) enhances sparsity and separation, overly aggressive sparsification risks fragmenting true clusters. Conversely, larger \( \beta \) retains more edges but may propagate noise. In practice, \( \beta \) is tuned to balance these effects—a process analogous to selecting regularization strength in ridge regression or dropout rates in neural networks.  

\textbf{Implications for Optimization:}  
The reweighting mechanism introduces a feedback loop between cluster assignments \( \bm{X} \) and graph structure \( \bm{W} \). As \( \bm{X} \) converges, \( \bm{W} \) adapts to reflect refined similarities, which in turn guides subsequent updates of \( \bm{X} \). This co-evolution is regularized by \( \beta \), ensuring gradual structural changes that promote stable convergence.  

In summary, the graph affinity update not only sparsifies connections but also implicitly regularizes the learning dynamics, fostering a robust equilibrium between assignment coherence and graph fidelity. This regularization is pivotal in practical settings where noise and model misspecification threaten to destabilize the clustering process.

\section*{Appendix D: Sub-optimality in Recursive Partitioning for Graph-Based Cut}
Recursive partitioning—iteratively applying a Normalized Cut (Ncut) to subdivide a graph into \( K \) partitions—is widely used for its simplicity and efficiency. However, it often leads to sub-optimal solutions when viewed from the global multi-way partition perspective. This short note elaborates why greedy two-way splitting may deviate from the globally optimal \( K \)-way cut, and illustrates the underlying mathematical and structural reasons.

Let \( \mathcal{G} = (\mathcal{V}, \mathcal{E}, \bm{W}) \) be an undirected graph with \( |\mathcal{V}| = N \), adjacency matrix \( \bm{W} \in \mathbb{R}_+^{N \times N} \), and degree matrix \( \bm{D} = \mathrm{diag}(\bm{W}\mathbf{1}_N) \). A \( K \)-way partition \( \{\mathcal{A}_k\}_{k=1}^K \) minimizes the \emph{Normalized Cut} (Ncut) objective:
\[
\mathrm{Ncut}(\{\mathcal{A}_k\}) = \sum_{k=1}^K \frac{\mathrm{cut}(\mathcal{A}_k, \mathcal{V}\setminus\mathcal{A}_k)}{\mathrm{vol}(\mathcal{A}_k)},
\]
where \( \mathrm{cut}(A,B) = \sum_{i \in A, j \in B} W_{ij} \) and \( \mathrm{vol}(A) = \sum_{i \in A} d_i \). Minimizing this objective is NP-hard for \( K > 2 \). As a practical alternative, many implementations use \textit{recursive bipartitioning}: first split \( \mathcal{G} \) into two subgraphs \( (\mathcal{A},\mathcal{B}) \) via a two-way Ncut, then recursively partition each subgraph until \( K \) clusters are obtained. Despite its convenience, this approach typically fails to achieve the globally optimal \( K \)-way solution.

\begin{itemize}
\item \textbf{Absence of Optimal Substructure.}  
  Multi-way Ncut lacks the property that its global optimum can be formed by combining locally optimal two-way cuts. Once the graph is divided into \( (\mathcal{A},\mathcal{B}) \), the boundary becomes irreversible. Even if \( \mathrm{Ncut}(\mathcal{A},\mathcal{B}) \) is locally minimized, this partition may block access to the true global optimum for \( K \)-way partitioning.

\item \textbf{NP-hardness of the \( K \)-way Cut.}  
  Exact global minimization for \( K > 2 \) is NP-hard. Polynomial-time methods, including recursive bipartitioning, rely on approximations. A greedy local cut locks in an irreversible partition boundary that may prove suboptimal in the final multi-way context.

\item \textbf{Misalignment with Clustering Structure.}  
  Real-world data often exhibit parallel communities rather than strict hierarchies. While a direct \( K \)-way partition (via the first \( K \) eigenvectors of \( \bm{L}_{\mathrm{sym}} \)) can isolate communities, forced two-way splits may prematurely merge clusters, increasing cross-cut edges unnecessarily.

\item \textbf{Spectral Limitations of Single Eigenvectors.}  
  Two-way Ncut relies on the Fiedler vector (second eigenvector of \( \bm{L}_{\mathrm{sym}} \)). For \( K > 2 \), higher eigenvectors encode critical structural information. By focusing on one eigenvector per bipartition, recursive splitting misses multi-community signals, leading to suboptimal merges that later steps cannot fully rectify.
\end{itemize}

\paragraph{Conclusion.}
\textbf{Recursive partitioning} introduces sub-optimality by imposing a sequential scheme on a global objective. Each local bipartition may appear optimal in isolation but need not align with the best \( K \)-way cut. This limitation is pronounced in non-hierarchical data or when multi-spectral components are essential. While recursive bipartitioning remains a heuristic for its simplicity, it can significantly deviate from the global optimum.

\clearpage
\section*{Appendix E: Algorithm Summary}
\begin{algorithm}[ht]
\caption{Falcon: Low-Resolution Mask Generation}
\label{alg:falcon_lowres}
\begin{algorithmic}[1]
\REQUIRE 
\begin{itemize}
    \item Vision transformer features.
    \item Parameters: number of clusters \( K \), number of iterations for fractional alternating cuts \( T_{\mathrm{cuts}} \), graph update scale \( \beta \), stability constant \( \epsilon \), etc.
\end{itemize}
\ENSURE Low-resolution segmentation masks \( \{M_k\} \).
\STATE \textbf{Graph Construction:} Compute the affinity matrix \( \bm{W} \in \mathbb{R}^{N \times N} \) from the transformer features and form the degree matrix 
\[
\bm{D} = \mathrm{diag}(d_1,\dots,d_N),\quad d_i = \sum_{j} {W}_{ij}.
\]
\STATE \textbf{Initialization:} Initialize soft mask assignment matrix \( \bm{X} \in \mathbb{R}_+^{N \times K} \) (with each row summing to 1) and auxiliary variables \( y_k \) (e.g., \( y_k = 1 \)).
\FOR{\( t=1,\dots,T_{\mathrm{cuts}} \)}
    \FOR{\( k=1,\dots,K \)}
        \STATE \( y_k \gets \sqrt{\frac{\bm{x}_k^\top \bm{W} \bm{x}_k}{\bm{x}_k^\top \bm{D} \bm{x}_k}} \)
    \ENDFOR
    \FOR{\( i=1,\dots,N \)}
        \FOR{\( k=1,\dots,K \)}
            \STATE \( \bm{X}_{ik} \gets \mathrm{Softmax}_k\!\left( \frac{\sum_j {W}_{ij} {X}_{jk} }{\sum_j {X}_{jk} {D}_{jj} + \epsilon} y_k\right) \)
        \ENDFOR
    \ENDFOR
    \STATE \textbf{(Optional)} Update \( \bm{W} \) as
    \[
    {W}_{ij} \gets {W}_{ij} \exp\!\left( -\frac{\left(1 - \frac{\langle \bm{X}_i, \bm{X}_j \rangle}{\|\bm{X}_i\|_2 \|\bm{X}_j\|_2}\right)^2}{\beta} \right).
    \]
\ENDFOR
\STATE Map the final assignments in \( \bm{X} \) to obtain the low-resolution masks \( \{M_k\} \).
\end{algorithmic}
\end{algorithm}

\begin{algorithm}[ht]
\caption{Falcon: High-Resolution Mask Refinement (DREAM)}
\label{alg:falcon_refine}
\begin{algorithmic}[1]
\REQUIRE 
\begin{itemize}
    \item RGB image \( \bm{x}_{\mathrm{rgb}} \in \mathbb{R}^{B\times 3\times H\times W} \), depth map \( \bm{x}_{\mathrm{depth}} \in \mathbb{R}^{B\times 1\times H\times W} \).
    \item Initial low-resolution mask \( M^{(0)} \) (from Algorithm~\ref{alg:falcon_lowres}).
    \item Parameters: \( \lambda \) (ELU scale), fusion weights \( \alpha_{\mathrm{rgb}}, \alpha_{\mathrm{depth}} \), number of refinement iterations \( T_{\mathrm{ref}} \), stability constant \( \epsilon \), etc.
\end{itemize}
\ENSURE Refined high-resolution segmentation mask \( M \in \mathbb{R}^{H \times W} \).
\STATE \textbf{Affinity Computation:} For feature map \( \bm{x} \), define the local affinity operator over an 8-connected neighborhood \( \mathcal{N} \):
\[
\mathcal{A}(\bm{x}) = \sum_{(i,j)\in\mathcal{N}} \left[ (\bm{x}_{i,j} - \bm{x}_c) + \lambda \mathrm{ELU}(\bm{x}_{i,j} - \bm{x}_c) \right],
\]
where \( \bm{x}_c \) is the feature at the center pixel. Normalize the affinity as:
\[
\mathcal{A}_{\mathrm{norm}}(\bm{x}) = -\frac{\mathcal{A}(\bm{x})}{\epsilon + 0.1 \sigma(\bm{x})},
\]
with \( \sigma(\bm{x}) \) denoting the standard deviation of \( \bm{x} \).
\STATE Compute \( \bm{A}_{\mathrm{rgb}} = \mathcal{A}_{\mathrm{norm}}(\bm{x}_{\mathrm{rgb}}) \), \( \bm{A}_{\mathrm{depth}} = \mathcal{A}_{\mathrm{norm}}(\bm{x}_{\mathrm{depth}}) \).
\STATE \textbf{Affinity Fusion:} Fuse the modalities:
\[
\bm{A} \gets \alpha_{\mathrm{rgb}} \bm{A}_{\mathrm{rgb}} + \alpha_{\mathrm{depth}} \bm{A}_{\mathrm{depth}}.
\]
\FOR{\( t=0,\dots,T_{\mathrm{ref}}-1 \)}
    \STATE Update the mask:
    \[
    M^{(t+1)} \gets \sum_{(i,j)\in\mathcal{N}} M^{(t)}_{i,j} \bm{A}_{i,j}.
    \]
\ENDFOR
\STATE \RETURN \( M^{(T_{\mathrm{ref}})} \).
\end{algorithmic}
\end{algorithm}


\end{document}